\PassOptionsToPackage{table}{xcolor}
\documentclass{samsung}

\usepackage{amsmath,amsfonts,bm}

\def\1{\bm{1}}

\DeclareMathAlphabet{\mathsfit}{\encodingdefault}{\sfdefault}{m}{sl}
\SetMathAlphabet{\mathsfit}{bold}{\encodingdefault}{\sfdefault}{bx}{n}

\usepackage{hyperref}
\usepackage{url}
\usepackage{graphicx}
\usepackage{subcaption}

\usepackage{booktabs}
\usepackage{tabularx}
\usepackage{tcolorbox}
\usepackage{titletoc}
\usepackage{multirow}
\usepackage{xcolor}
\usepackage{colortbl}
\usepackage{amssymb,amsmath,amsthm}
\usepackage[capitalize,noabbrev]{cleveref}
\usepackage{enumitem}
\usepackage{wrapfig}
\BeforeBeginEnvironment{wrapfigure}{\par}
\usepackage{algorithm}
\usepackage{algpseudocode}

\newcommand{\com}[1]{\tiny$\pm$#1}

\theoremstyle{definition}
\newtheorem{theorem}{Theorem}[section]

\newtheorem{lemma}[theorem]{Lemma}

\theoremstyle{definition}

\newtheorem{assumption}[theorem]{Assumption}
\crefname{assumption}{Assumption}{Assumptions}
\newtheorem{remark}[theorem]{Remark}

\algrenewcommand\algorithmicensure{\textbf{Output:}}

\title{Reasoning-Preserving Fine-Tuning of Post-RL LLMs with Null-Basis LoRA}

\author[1,2]{Wenzhi Fang}
\author[1]{Nicholas Tzou}
\author[1]{Lazar Valkov}
\author[1]{Srinivas Chappidi}
\affiliation[1]{Samsung Research America, Mountain View, CA, USA}
\affiliation[2]{Purdue University, West Lafayette, IN, USA}
\correspondence{\email{fang375@purdue.edu}, \email{n.tzou@samsung.com}}

\definecolor{row-highlight}{RGB}{205, 220, 236}

\abstract{
Reinforcement learning (RL)-based post-training has become an effective approach for eliciting reasoning capabilities in large language models (LLMs). However, adapting post-RL models to new knowledge domains or behaviors through subsequent supervised fine-tuning (SFT) can severely overwrite these capabilities. 
Existing approaches mitigate such forgetting through experience replay, specialized initialization, or constrained optimization using gradient projection, but either provide limited preservation or incur substantial training overhead. 
Our analysis shows that reasoning activations concentrate in low-dimensional subspaces, leaving substantial null-space capacity for adaptation, and that the corresponding approximate null spaces can be reliably estimated from a modest number of examples. 
Motivated by these observations, we propose Null-Basis Low-Rank Adaptation (NB-LoRA), a parameter-efficient method for adapting post-RL LLMs while preserving their acquired reasoning ability. We formulate reasoning retention as a layer-wise hidden-state preservation constraint and construct a fixed approximate null basis from reasoning activations. LoRA updates are then reparameterized through this basis, enforcing the preservation constraint throughout fine-tuning. 
Extensive experiments across multiple RL-trained LLMs and diverse downstream tasks show that NB-LoRA matches standard LoRA in adaptation performance, maintains reasoning accuracy near pre-fine-tuning levels, and generalizes this preservation to held-out reasoning benchmarks.
}

\begin{document}

\maketitle


\section{Introduction}
Reinforcement learning (RL) based post-training has become an effective approach for eliciting reasoning capabilities in large language models (LLMs)~\citep{shao2024deepseekmath,yang2024qwen2}. However, most current reasoning-oriented RL methods primarily exploit capabilities already encoded in the base model by reshaping its output distribution, and typically fall short of instilling new knowledge. Moreover, the resulting models often generate lengthy reasoning traces, even for tasks where concise responses would suffice. Subsequent supervised fine-tuning (SFT) therefore remains necessary to inject new knowledge or adapt the model's response style in some target tasks. Yet these RL-acquired capabilities are fragile under such adaptation: RL only mildly perturbs the base model, staying close to it in distribution~\citep{shenfeld2025rlrazor}, whereas the subsequent SFT is not similarly constrained and can drift far from that distribution, easily overwriting the RL-shaped reasoning behavior~\citep{niu2026nondecoupling}. 


Retaining previously acquired capabilities while adapting to new data is a classic problem of catastrophic forgetting, for which continual learning provides several natural solutions~\citep{kirkpatrick2017overcoming}. Conventional approaches such as experience replay~\citep{rolnick2019experience} and null-space update projection~\citep{wang2021training}, however, can be difficult to apply to reasoning-oriented LLM adaptation. Experience replay requires carefully balancing prior reasoning data with new-task data: insufficient replay may fail to preserve reasoning, whereas excessive replay may hinder adaptation to the new task. Null-space update projection methods instead constrain full-model updates, introducing substantial computational and memory overhead at LLM scale. More recent parameter-efficient approaches mitigate forgetting through low-rank adaptation (LoRA) and carefully chosen LoRA initialization~\citep{wang2025milora}. However, due to unconstrained update, these methods may still overwrite the reasoning behavior acquired during RL post-training, which can be \emph{particularly sensitive} to further adaptation. This raises the central question of this paper: 

\emph{How can we efficiently adapt a RL-trained LLM to new tasks while preserving its acquired reasoning ability?}


\subsection{Contribution}
To address this, we propose Null-Basis LoRA (NB-LoRA), a parameter-efficient fine-tuning method that reparameterizes LoRA updates through a fixed null basis of reasoning hidden states. The idea builds on the fact that if an update does not change a layer's outputs on the hidden states produced by reasoning examples, the reasoning computation carried by those states is left intact. Specifically, we first compute, once and offline, a basis $\mathbf{N}_s$ for the (approximate) null space of the reasoning hidden-state matrix $\mathbf{H}$, and reparameterize the update as $\Delta \mathbf{W} = \mathbf{B}\mathbf{C}\mathbf{N}_s^{\top}$ with $\mathbf{N}_s$ frozen and the low-rank factors $\mathbf{B}$ and $\mathbf{C}$ trainable (Figure~\ref{fig:overview}). Any $\mathbf{B}$ and $\mathbf{C}$ satisfy the preservation constraint $\mathbf{H}\Delta\mathbf{W}^{\top}=\mathbf{0}$ exactly, enabling unconstrained optimization with standard optimizers.

Our main contributions are summarized as follows.
\begin{itemize}[leftmargin=*, itemsep=2pt, topsep=2pt, parsep=0pt]
    \item We formulate post-RL reasoning retention as a layer-wise hidden-state preservation constraint and introduce Null-Basis LoRA (NB-LoRA), a null-basis low-rank reparameterization that satisfies this constraint by construction while retaining the optimizer simplicity of standard LoRA.

    \item We show that reasoning activations in RL-trained LLMs occupy a low-dimensional subspace despite long reasoning traces, leaving a substantial approximate null space for adaptation that can be estimated from only a small set of reasoning examples.

    \item We empirically demonstrate that NB-LoRA preserves RL-acquired reasoning ability while maintaining strong downstream adaptation performance, generalizes retention to held-out reasoning benchmarks, and substantially improves training efficiency over existing works.
\end{itemize}

\begin{figure}
    \centering
    \includegraphics[width=0.95\textwidth]{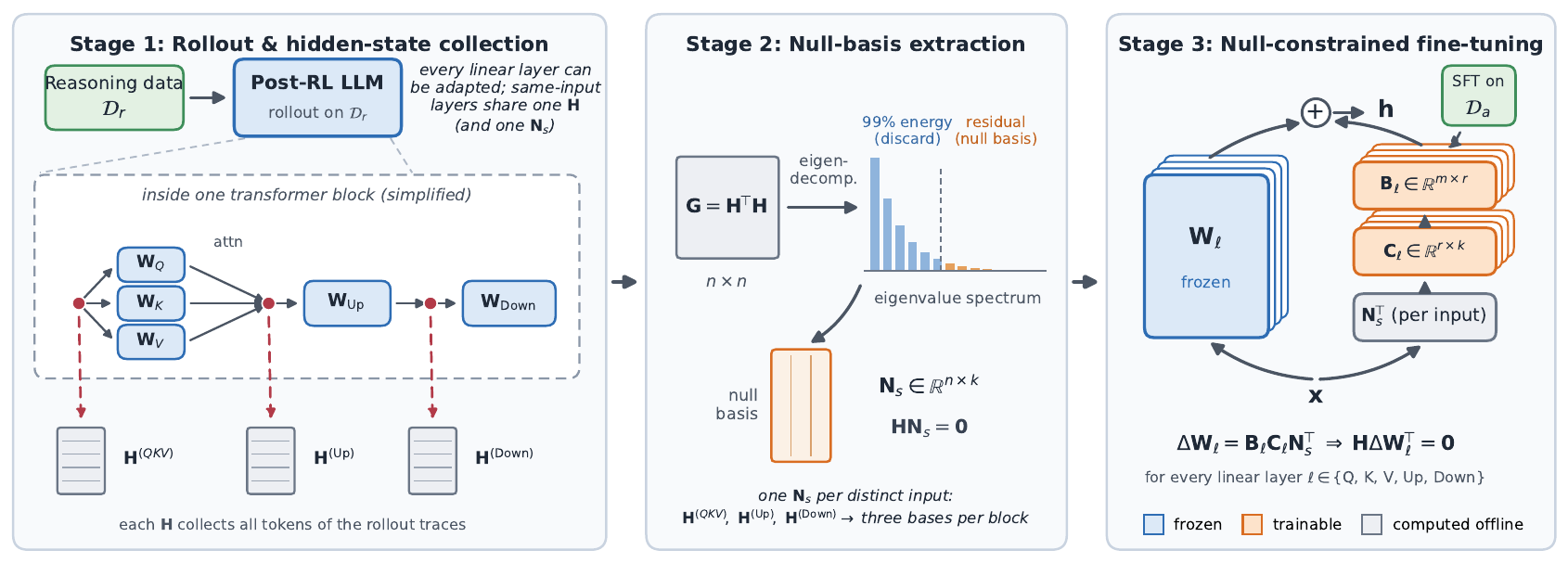}
    \caption{Overview of NB-LoRA. \textbf{Stage 1:} We use the post-RL model to generate rollouts on reasoning data $\mathcal{D}_r$ and collect the input hidden states
$\mathbf{H}$ of each target linear layer. \textbf{Stage 2:} The eigendecomposition of
$\mathbf{H}^{\top}\mathbf{H}$ yields an approximate null basis
$\mathbf{N}_s$ from the eigenvectors outside the top $\eta$ (i.e., 99\%) spectral energy,
so that $\mathbf{H}\mathbf{N}_s=\mathbf{0}$. \textbf{Stage 3:} The update is
reparameterized as $\Delta\mathbf{W}=\mathbf{B}\mathbf{C}\mathbf{N}_s^{\top}$, where $\mathbf{B}$ and $\mathbf{C}$ are trained on the adaptation data
$\mathcal{D}_a$, while $\mathbf{N}_s$ remains frozen.}
\label{fig:overview}
\vspace{-1.5em}
\end{figure}

\section{Related Work}

\paragraph{Interaction between SFT and RL in LLM post-training.}
RL has been widely adopted in LLM post-training to improve reasoning capabilities~\citep{yang2024qwen2,shao2024deepseekmath}. Typical post-training pipelines combine SFT and RL because they provide complementary learning signals: SFT introduces expert behavior and new knowledge, whereas online RL refines the model through exploration. ReFT follows the conventional SFT-then-RL pipeline~\citep{trung2024reft}, while DeepSeek-R1 employs a multistage SFT-RL-SFT-RL pipeline~\citep{deepseekai2025r1}. Recently, \citet{ma2026relift} proposed ReLIFT, which interleaves RL with online SFT on problems beyond the model's capabilities. Because its SFT data are collected during RL, the two training objectives are tightly coupled. In general, however, subsequent SFT may target a different task and reduce the reward acquired through RL~\citep{niu2026nondecoupling}. We study this RL-then-SFT setting, aiming to adapt the model without overwriting its RL-acquired reasoning.

\paragraph{Mitigating catastrophic forgetting.}
Existing approaches to catastrophic forgetting include regularization, replay, on-policy distillation, and update constraints. Elastic weight consolidation penalizes changes to parameters important for prior tasks~\citep{kirkpatrick2017overcoming}, while replay mixes stored~\citep{rolnick2019experience} or model-generated examples~\citep{sun2020lamol} with new-task data. SDFT uses a demonstration-conditioned copy of the model to supervise on-policy trajectories, reducing distribution mismatch and forgetting, but is limited by the model's in-context learning ability and can fail under distribution shift~\citep{shenfeld2026selfdistillation,wang2026denser}. Constraint-based methods prevent loss increases on prior tasks~\citep{lopezpaz2017gem} or project gradients away from directions affecting previous outputs~\citep{farajtabar2020orthogonal}. More closely related to our work, NSCL projects updates into activation-derived null spaces at each step~\citep{wang2021training}. Unlike these soft, replay-based, or per-step constraints, the proposed NB-LoRA encodes a fixed reasoning-preservation constraint directly into the adapter.

\paragraph{Parameter-efficient fine-tuning.}
Parameter-efficient fine-tuning (PEFT) adapts a frozen backbone by training only a small number of additional parameters~\citep{houlsby2019parameter,li2021prefix}. A prominent PEFT approach, LoRA, parameterizes weight updates with trainable low-rank factors~\citep{hu2022lora}. Several extensions apply LoRA to continual learning. O-LoRA assigns sequential tasks to orthogonal low-rank subspaces~\citep{wang2023orthogonal}, while LB-CL transfers task-relevant low-rank parameters and projects gradients to maintain orthogonality with prior task subspaces~\citep{qiao2024learn}. Merge Before Forget instead orthogonally initializes and continually merges task updates into a single LoRA, maintaining constant memory across tasks~\citep{qiao2026merge}. These methods require prior tasks to be adapted through LoRA, whereas our setting protects reasoning already embedded in an RL-trained backbone. LoRA-Null and MiLoRA initialize adapters using null or minor singular subspaces~\citep{tang2026put,wang2025milora}, but do not constrain subsequent updates. Our proposed NB-LoRA instead enforces preservation throughout training by null-basis-based reparameterization. 

\section{Null-Basis-based Parameter-Efficient Fine-Tuning}

\subsection{Problem Setup}
We consider a reasoning-capable LLM obtained via RL-based post-training. Let $\mathcal{D}_r$ denote the dataset used in the RL stage, whose induced capability we wish to preserve, and let $\mathcal{D}_a$ denote an adaptation dataset for subsequent SFT, such as knowledge injection or concise-response tuning. These two cases are representative of common post-RL adaptation goals: one seeks to augment the model's knowledge, while the other seeks to modify output behavior on target domains without altering its reasoning competence. Our goal is to optimize the model on $\mathcal{D}_a$ while minimizing interference with reasoning performance on $\mathcal{D}_r$.

We focus on adapting linear layers in the model, because transformer-based LLMs are built largely around linear modules, particularly the attention and MLP projections. For a linear layer with weight matrix $\mathbf{W} \in \mathbb{R}^{m \times n}$, standard fine-tuning learns an update $\Delta \mathbf{W}$ and replaces $\mathbf{W}$ with $\mathbf{W} + \Delta \mathbf{W}$. Such unconstrained updates can modify the layer outputs for both adaptation and reasoning examples. 
The objective of this work is to constrain the update $\Delta \mathbf{W}$ to a restricted subspace in a parameter- and training-efficient manner, such that the update does not alter the model outputs on reasoning data.

\subsection{Hidden-State Preservation Constraint}
Let $\mathbf{H} \in \mathbb{R}^{N \times n}$ denote a matrix of hidden states collected from $\mathcal{D}_r$ for the input of a target linear layer, where each row of $\mathbf{H}$ corresponds to one token representation and $N$ is the total number of collected tokens. To preserve the pre-fine-tuning transformation of these reasoning states, we require
\begin{equation}
\mathbf{H}(\mathbf{W} + \Delta \mathbf{W})^{\top} = \mathbf{H}\mathbf{W}^{\top}.
\label{eq:preserve}
\end{equation}
Equation~\eqref{eq:preserve} is equivalent to
$\mathbf{H}\Delta \mathbf{W}^{\top} = \mathbf{0}.$
Therefore, a sufficient condition for preserving the layer outputs on reasoning states is that every column of $\Delta \mathbf{W}^{\top}$ lies in the null space of $\mathbf{H}$.

This constraint has two desirable properties. First, it provides layer-wise and fine-grained preservation signals by constraining each adapted linear layer with hidden states observed on reasoning data, thereby directly guiding the search space of $\Delta \mathbf{W}$ toward directions that preserve reasoning-related capabilities. Second, in transformer blocks, multiple projections such as $\mathbf{Q}$, $\mathbf{K}$, and $\mathbf{V}$ share the same input hidden states, so the corresponding null space can be computed once and reused across these projections, substantially reducing the computational cost of constructing the constraint.

\subsection{Feasibility of Null-Space Constraints}\label{sec:sub:subspace_large_enough}

Two key questions arise when applying null-space constraints to reasoning models.
\emph{1) Does $\mathbf{H}$ admit a sufficiently large approximate null space?}
Long reasoning traces produce many token-level observations ($N \gg n$), potentially increasing $\operatorname{rank}(\mathbf{H})$ toward $n$ and reducing its null-space dimension,
$\dim(\operatorname{Null}(\mathbf{H})) = n - \operatorname{rank}(\mathbf{H})$.
\emph{2) Can this null space be effectively estimated using only a limited number of reasoning examples?}
Constructing $\mathbf{H}$ from a large reasoning corpus can be computationally expensive.
To answer these questions, we empirically analyze the spectrum and null-space structure of $\mathbf{H}$ using hidden states extracted from Qwen2.5-3B~\citep{yang2024qwen2} after RL post-training on the MATH-lighteval dataset~\citep{hendrycks2021measuring}. Notably, the null space of $\mathbf{H}$ can be obtained from the eigendecomposition of the much smaller $n \times n$ matrix $\mathbf{H}^{\top}\mathbf{H}$. Since hidden states are noisy and $\mathbf{H}$ is rarely exactly low rank, we define an approximate null space following the approach of~\citet{wang2021training}: we retain the principal directions that capture 99\% of the spectral energy and treat the remaining directions as the approximate null-space basis.

\begin{table}[h]
\centering
\caption{Approximate null-space dimensions ($k$) of representative Qwen2.5-3B layers under a 99\% energy-retention threshold. Percentages are relative to the total dimension.}
\label{tab:null_space_dim}
\resizebox{\textwidth}{!}{
\begin{tabular}{lcrrrrrrr}
\toprule
Module & Total Dimension & Layer 0 & Layer 6 & Layer 12 & Layer 18 & Layer 24 & Layer 30 & Average (\%) \\
\midrule
QKV  & 2048  & 1378 & 765  & 755  & 515  & 407  & 393  & 702 (34.3\%) \\
Down & 11008 & 7028 & 7343 & 5474 & 2272 & 1701 & 5940 & 4960 (45.1\%) \\
Up   & 2048  & 1037 & 1138 & 574  & 325  & 259  & 251  & 597 (29.2\%) \\
\bottomrule
\end{tabular}
}
\end{table}

\textbf{Spectrum analysis.} We collect hidden states (i.e., $\mathbf{H}$) from all the tokens in the reasoning traces of MATH-lighteval training samples. Figure~\ref{fig:normalized_activation_spectrum} shows the normalized eigenvalue spectrum of $\mathbf{H}^{\top}\mathbf{H}$ for all the linear modules across layers 0, 6, 12, 18, 24, and 30 of Qwen2.5-3B. In all cases, the spectrum decays rapidly by several orders of magnitude within the first few hundred indices, indicating that the hidden states of reasoning traces span a proper subspace of the full representation space. We report the dimensions of the approximate null spaces (99\% energy threshold) of the hidden-state matrices for each module in Table~\ref{tab:null_space_dim}, which confirms that substantial null-space capacity remains available for further adaptation.

\begin{figure}[t]
\centering
\includegraphics[width=\linewidth]{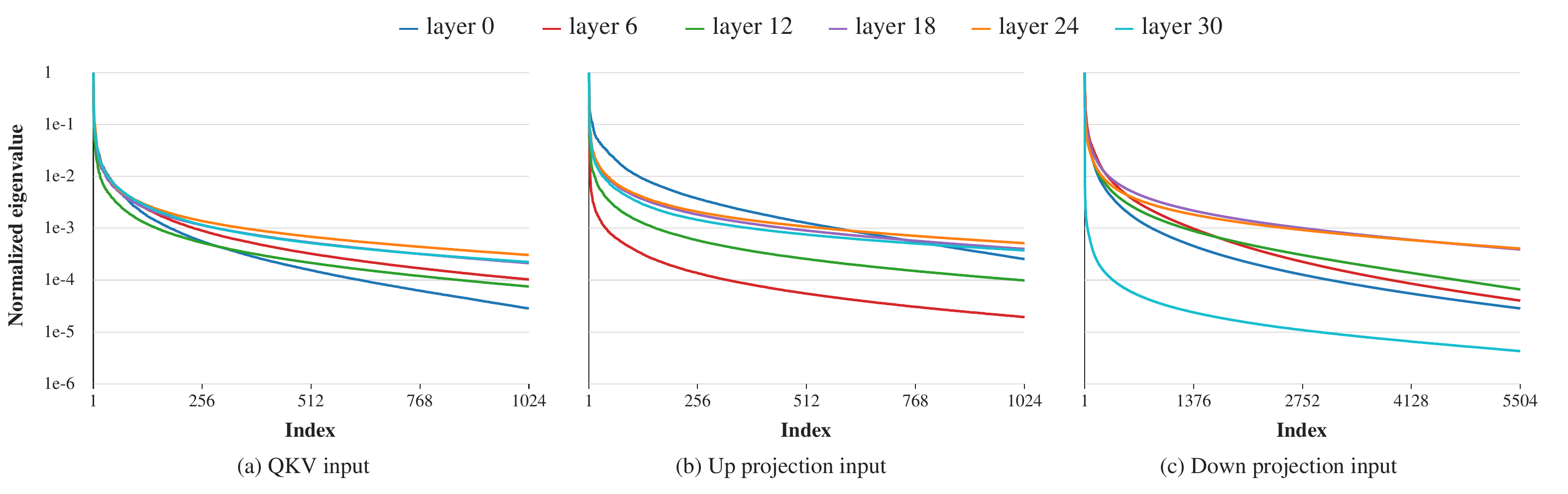}
\caption{Normalized eigenvalue spectra of $\mathbf{H}^{\top}\mathbf{H}$ across representative Qwen2.5-3B layers, shown on a log scale. 
}
\label{fig:normalized_activation_spectrum}
\vspace{-1em}
\end{figure}
\begin{figure}[h]
\centering
\includegraphics[width=\linewidth]{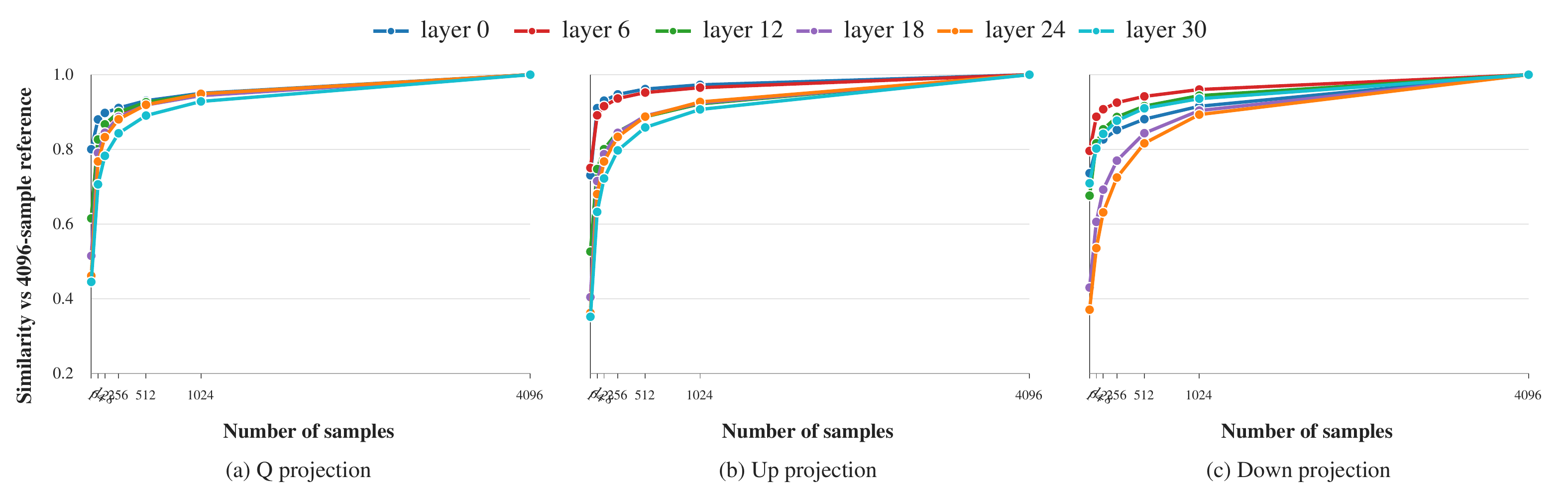}
\caption{Symmetric projection similarity between null spaces estimated using varying numbers of examples and the 4096-sample reference null space across representative Qwen2.5-3B layers.}
\label{fig:symmetric_projection_similarity_vs_samples_by_module}
\vspace{-1em}
\end{figure}

\textbf{Null-space Structure Analysis.} 
Since reasoning traces exhibit similar structural patterns~\citep{jiang2025makes,sun2026llm}, we hypothesize that the null-space structure
can be reliably estimated from a moderate number of examples. To verify this, we compare approximate null spaces estimated using different numbers of examples by measuring the principal angles between them. Specifically, we use the null space estimated from 4096 examples as a reference and evaluate its similarity to those estimated using fewer examples. Let $\mathcal{N}_1, \mathcal{N}_2 \subset \mathbb{R}^n$
denote two subspaces with dimensions $k_1 \leq k_2$. 
The principal angles are defined recursively, for $j=1,\ldots,k_1$, as
\begin{equation}
\theta_j =
\min_{\substack{
\mathbf{x}_j \in \mathcal{N}_1,\,
\mathbf{y}_j \in \mathcal{N}_2 \\
\mathbf{x}_j \perp \mathbf{x}_i^*,\,
\mathbf{y}_j \perp \mathbf{y}_i^*,\ \forall i<j
}}
\arccos\left(
\frac{\mathbf{x}_j^\top\mathbf{y}_j}
{\|\mathbf{x}_j\|\,\|\mathbf{y}_j\|}
\right),
\qquad j=1,\ldots,k_1,
\end{equation}
where $(\mathbf{x}_i^*, \mathbf{y}_i^*)$ denotes an optimizing pair for $\theta_i$. 
Smaller principal angles $\{\theta_1, \ldots, \theta_{k_1}\}$ indicate closer alignment between the two null spaces.
This recursive definition gives the \emph{geometric characterization} of principal angles. In practice, they can be computed more conveniently from the orthonormal bases of the two subspaces. Following~\citet{knyazev2012principal}, let $\mathbf{N}_1$ and $\mathbf{N}_2$ be orthonormal bases of $\mathcal{N}_1$ and $\mathcal{N}_2$, respectively:
\begin{equation}
    \mathbf{N}_1^\top \mathbf{N}_2
=
\mathbf{U}\,\mathrm{diag}(\sigma_1,\ldots,\sigma_{k_1})\,\mathbf{V}^\top,
\qquad
\theta_i = \arccos(\sigma_i), \quad i=1,\ldots,k_1.
\end{equation}
Appendix~\ref{app:principal_angles} provides a derivation connecting the definition of principal angles to this SVD-based computation.

In Figure~\ref{fig:symmetric_projection_similarity_vs_samples_by_module}, we plot the symmetric projection similarity, a normalized sum of the squared cosines of the principal angles between each estimated null space and the 4096-sample reference null space, defined as 
$\frac{1}{\sqrt{k_1 k_2}} \sum_{i=1}^{k_1} \cos^2 \theta_i.$ The similarity increases rapidly and largely stabilizes beyond 512 samples for most layers and modules, indicating that the estimated null-space structure becomes stable with a moderate number of examples.

\subsection{Proposed Null-Basis Reparameterization}

\Cref{sec:sub:subspace_large_enough} demonstrates the feasibility of using the null-space constraint to preserve reasoning capabilities. Building on this result, we construct a fixed null-space basis and
use it to reparameterize the low-rank update $\Delta\mathbf{W}$. This reparameterization structurally constrains every update to the estimated approximate null space, thereby limiting changes to the transformations of reasoning hidden states throughout fine-tuning.

Concretely, let $\mathbf{N}_s \in \mathbb{R}^{n \times k}$ denote the
orthonormal basis of the null space of $\mathbf{H}$ (i.e., $\mathbf{H}\mathbf{N}_s = \mathbf{0}$), where $k$ is the null-space dimension. We then parameterize the update as
\begin{equation}
\Delta \mathbf{W} = \mathbf{B}\mathbf{C}\mathbf{N}_s^{\top},
\label{eq:deltaw}
\end{equation}
with trainable matrices $\mathbf{B} \in \mathbb{R}^{m \times r}$ and $\mathbf{C} \in \mathbb{R}^{r \times k}$, where $r$ is a chosen adaptation rank. We thus have
\begin{equation}
\mathbf{H}\Delta \mathbf{W}^{\top}
= \mathbf{H}\mathbf{N}_s\mathbf{C}^{\top}\mathbf{B}^{\top}
= \mathbf{0}, \quad \forall ~ \mathbf{B}, \mathbf{C},
\end{equation}
which means the preservation constraint is satisfied by construction. We optimize the trainable factors $\mathbf{B}$ and $\mathbf{C}$ on the adaptation dataset $\mathcal{D}_a$ using the standard SFT objective. The overall procedure is summarized in \Cref{alg:proposed_method}.

The number of trainable parameters is $mr + rk$ for this reparameterization, which is smaller than the $mr + rn$ of standard LoRA as the null-space dimension $k < n$. As the preservation constraint in \eqref{eq:preserve} is enforced structurally through the parameterization, $\mathbf{B}$ and $\mathbf{C}$ are unconstrained parameters and can be updated by any standard optimizer (e.g., SGD~\citep{bottou2010large}, Adam~\citep{kingma2014adam}) without requiring any projection or constraint enforcement at each iteration.
In addition, since $\mathbf{N}_s$ has orthonormal columns, composing through it preserves the gradient Lipschitz constant of the loss, and our method enjoys the same $O(1/\log(T))$ convergence rate to a stationary point as standard LoRA under the same smoothness assumptions, where $T$ denotes the number of optimization steps~\citep{mu2026convergence}. We provide a formal statement in~\Cref{app:convergence}.

\begin{remark}
    Mathematically, one may simplify \eqref{eq:deltaw} into $\Delta \mathbf{W} = \widetilde{\mathbf{B}}\mathbf{N}_s^{\top}$ with $\widetilde{\mathbf{B}} \in \mathbb{R}^{m \times k}$. We keep the factorization $\mathbf{B}\mathbf{C}\mathbf{N}_s^{\top}$ because it provides a tunable bottleneck through rank $r$, which can substantially reduce trainable parameters when $k$ is large.
    This reparameterization can be interpreted as a constrained low-rank update. Standard LoRA sets $\Delta \mathbf{W} = \mathbf{B}\mathbf{A}$ with no restriction on $\mathbf{A}$. In contrast, our method sets
    $\mathbf{A} = \mathbf{C}\mathbf{N}_s^{\top},$ so that the row space of $\mathbf{A}$ is restricted to the null space of the reasoning-state matrix. The method therefore retains the parameter-efficiency benefits of low-rank adaptation while structurally limiting changes to layer outputs on reasoning-related hidden states.
\end{remark}

\begin{algorithm}[t]
\caption{Null-Basis Low-Rank Adaptation (NB-LoRA)\label{alg:proposed_method}}
\begin{algorithmic}[1]
\Require Post-RL model with target weight matrices $\{\mathbf{W}_\ell\}$; reasoning data $\mathcal{D}_r$; adaptation data $\mathcal{D}_a$; energy threshold $\epsilon$ (e.g., $0.99$); rank $r$
\State \textbf{Stage 1: Rollout}
\State Roll out the model on a subset of reasoning data $\mathcal{D}_r$ and collect input hidden states $\mathbf{H}_\ell$ for each target layer $\ell$ across all tokens
\State \textbf{Stage 2: Null-basis computation (offline, once)}
\For{each target layer $\ell$}
    \State Form $\mathbf{G}_\ell = \mathbf{H}_\ell^{\top}\mathbf{H}_\ell$ and compute its eigendecomposition
    \State Discard the leading eigenvectors covering $\epsilon$ of the spectral energy; stack the remaining eigenvectors as $\mathbf{N}_{s,\ell} \in \mathbb{R}^{n \times k_\ell}$
\EndFor
\State \textbf{Stage 3: Constrained adaptation}
\State Freeze $\mathbf{W}_\ell$ and $\mathbf{N}_{s,\ell}$; introduce trainable matrices $\mathbf{B}_\ell \in \mathbb{R}^{m \times r}$ and $\mathbf{C}_\ell \in \mathbb{R}^{r \times k_\ell}$, parameterizing $\Delta\mathbf{W}_\ell = \mathbf{B}_\ell\mathbf{C}_\ell\mathbf{N}_{s,\ell}^{\top}$
\State Optimize $\{\mathbf{B}_\ell, \mathbf{C}_\ell\}$ on $\mathcal{D}_a$ using the standard SFT objective
\Ensure Adapted model with $\mathbf{W}_\ell + \mathbf{B}_\ell\mathbf{C}_\ell\mathbf{N}_{s,\ell}^{\top}$ for each target layer $\ell$
\end{algorithmic}
\end{algorithm}

\section{Experiments}
We consider two reasoning models, Qwen2.5-3B and Qwen2.5-1.5B, which are trained via GRPO on MATH-lighteval~\citep{hendrycks2021measuring} and GSM8K~\citep{cobbe2021training}, respectively. In SFT stage, we adapt these models to a tool-use task~\citep{tang2023toolalpaca} and eight target tasks from the Commonsense Reasoning benchmark~\citep{hu2023llm}. 
After fine-tuning on the target tasks, we evaluate reasoning retention on both the anchor benchmark and several held-out benchmarks whose data are not used to construct the preservation constraint, including MATH-500, AGIEval-EN-MATH~\citep{zhong2024agieval}, and MMLU-STEM~\citep{hendrycks2020measuring}. More details about the hyperparameters and datasets are provided in~\Cref{app:appendix_experiment_details}.
\begin{table}[t]
\centering
\small
\caption{\textbf{Adaptation on new target tasks and retention on the reasoning benchmark.} Reasoning is evaluated on the benchmark used for RL post-training: MATH-lighteval for Qwen2.5-3B and GSM8K for Qwen2.5-1.5B. Our proposed method matches Vanilla LoRA on target task accuracy while retaining the reasoning performance.}
\label{tab:target_math_tradeoff_two_models}
\setlength{\tabcolsep}{2.2pt}
\resizebox{\textwidth}{!}{%
\begin{tabular}{l c c c c c c c c c c}
\toprule
\rowcolor{white}
\textbf{Method}
& \textbf{ARC-C}
& \textbf{ARC-E}
& \textbf{BoolQ}
& \textbf{HellaSwag}
& \textbf{OpenBookQA}
& \textbf{PIQA}
& \textbf{SocialIQA}
& \textbf{WinoGrande}
& \textbf{ToolUse}
& \textbf{Avg.} \\
\midrule
\rowcolor{white}
\multicolumn{1}{l}{\textbf{Target}} & \multicolumn{10}{c}{\textbf{Qwen2.5-3B}} \\
\midrule
\rowcolor{gray!10}
Vanilla LoRA & 0.824 & \textbf{0.948} & 0.715 & \textbf{0.947} & 0.868 & 0.861 & \textbf{0.814} & \textbf{0.857} & \textbf{0.790} & \textbf{0.847} \\
LoRA-Null    & 0.823 & 0.924 & 0.696 & 0.944 & 0.854 & 0.875 & 0.810 & 0.838 & 0.758 & 0.836 \\
\rowcolor{gray!10}
MiLoRA       & 0.773 & 0.926 & 0.720 & 0.946 & \textbf{0.884} & \textbf{0.878} & 0.754 & 0.839 & 0.710 & 0.826 \\
NSCL         & 0.794 & 0.937 & 0.706 & 0.945 & 0.880 & 0.870 & 0.812 & 0.830 & 0.774 & 0.839 \\
\rowcolor{gray!10}
Exp. Replay  & 0.815 & 0.935 & \textbf{0.722} & 0.944 & \textbf{0.884} & 0.868 & 0.810 & 0.849 & 0.774 & 0.845 \\
SDFT         & 0.643 & 0.856 & 0.601 & 0.702 & 0.728 & 0.730 & 0.681 & 0.698 & 0.645 & 0.698 \\
\rowcolor{row-highlight}
\textbf{Ours} & \textbf{0.826} & 0.941 & 0.713 & 0.945 & 0.882 & 0.873 & 0.811 & 0.847 & 0.774 & 0.846 \\
\midrule
\rowcolor{white}
\multicolumn{1}{l}{\textbf{Reasoning}} & \multicolumn{10}{c}{\textbf{Qwen2.5-3B} (Pre-SFT: $63.8\%$; MATH-lighteval)} \\
\midrule
\rowcolor{gray!10}
Vanilla LoRA & 0.089 & 0.060 & 0.013 & 0.245 & 0.064 & 0.000 & 0.162 & 0.033 & 0.582 & 0.139\com{0.183} \\
LoRA-Null    & 0.387 & 0.477 & 0.622 & 0.577 & 0.493 & 0.556 & 0.624 & 0.612 & 0.604 & 0.550\com{0.082} \\
\rowcolor{gray!10}
MiLoRA       & 0.160 & 0.103 & 0.357 & 0.000 & 0.102 & 0.204 & 0.133 & 0.095 & 0.272 & 0.158\com{0.106} \\
NSCL         & \textbf{0.637} & 0.635 & \textbf{0.637} & 0.635 & 0.635 & \textbf{0.638} & 0.638 & \textbf{0.631} & 0.614 & \textbf{0.633}\com{0.008} \\
\rowcolor{gray!10}
Exp. Replay  & 0.618 & 0.602 & 0.625 & 0.592 & 0.613 & 0.556 & 0.632 & 0.622 & \textbf{0.620} & 0.609\com{0.023} \\
SDFT         & 0.635 & 0.634 & 0.633 & 0.633 & \textbf{0.637} & 0.634 & \textbf{0.639} & \textbf{0.631} & 0.612 & 0.632\com{0.008} \\
\rowcolor{row-highlight}
\textbf{Ours} & 0.635 & \textbf{0.638} & 0.634 & \textbf{0.637} & 0.629 & 0.636 & 0.633 & \textbf{0.631} & 0.617 & 0.632\com{0.006} \\
\midrule
\rowcolor{white}
\multicolumn{1}{l}{\textbf{Target}} & \multicolumn{10}{c}{\textbf{Qwen2.5-1.5B}} \\
\midrule
\rowcolor{gray!10}
Vanilla LoRA & 0.753 & 0.906 & 0.686 & \textbf{0.918} & 0.842 & 0.852 & 0.777 & 0.785 & \textbf{0.661} & \textbf{0.798} \\
LoRA-Null    & 0.734 & 0.900 & \textbf{0.692} & 0.913 & 0.854 & \textbf{0.854} & \textbf{0.783} & 0.774 & 0.177 & 0.742 \\
\rowcolor{gray!10}
MiLoRA       & 0.753 & 0.908 & 0.690 & 0.913 & 0.836 & 0.845 & 0.780 & \textbf{0.789} & 0.000 & 0.724 \\
NSCL         & 0.695 & 0.893 & 0.689 & 0.908 & \textbf{0.856} & 0.843 & 0.774 & 0.765 & 0.065 & 0.721 \\
\rowcolor{gray!10}
Exp. Replay  & 0.758 & 0.902 & 0.689 & 0.914 & 0.844 & 0.836 & 0.753 & 0.783 & 0.000 & 0.720 \\
SDFT         & 0.066 & 0.103 & 0.065 & 0.021 & 0.058 & 0.403 & 0.014 & 0.183 & 0.161 & 0.119 \\
\rowcolor{row-highlight}
\textbf{Ours} & \textbf{0.771} & \textbf{0.912} & 0.683 & 0.914 & 0.846 & 0.843 & 0.767 & 0.763 & 0.629 & 0.792 \\
\midrule
\rowcolor{white}
\multicolumn{1}{l}{\textbf{Reasoning}} & \multicolumn{10}{c}{\textbf{Qwen2.5-1.5B} (Pre-SFT: $80.1\%$; GSM8K)} \\
\midrule
\rowcolor{gray!10}
Vanilla LoRA & 0.162 & 0.020 & 0.304 & 0.017 & 0.092 & 0.021 & 0.054 & 0.081 & 0.750 & 0.167\com{0.237} \\
LoRA-Null    & 0.435 & 0.657 & 0.666 & 0.718 & 0.606 & 0.661 & 0.544 & 0.608 & 0.508 & 0.600\com{0.090} \\
\rowcolor{gray!10}
MiLoRA       & 0.406 & 0.575 & 0.497 & 0.561 & 0.663 & 0.286 & 0.547 & 0.514 & 0.531 & 0.509\com{0.108} \\
NSCL         & \textbf{0.801} & 0.794 & 0.769 & 0.773 & 0.801 & 0.782 & 0.793 & 0.794 & 0.788 & 0.788\com{0.011} \\
\rowcolor{gray!10}
Exp. Replay  & 0.775 & 0.775 & 0.648 & 0.654 & 0.643 & 0.607 & 0.388 & 0.756 & 0.262 & 0.612\com{0.177} \\
SDFT         & 0.799 & 0.792 & 0.801 & 0.798 & \textbf{0.811} & \textbf{0.801} & \textbf{0.795} & 0.792 & 0.799 & 0.799\com{0.006} \\
\rowcolor{row-highlight}
\textbf{Ours} & 0.799 & \textbf{0.800} & \textbf{0.804} & \textbf{0.801} & 0.797 & 0.795 & 0.792 & \textbf{0.809} & \textbf{0.803} & \textbf{0.800}\com{0.005} \\
\bottomrule
\end{tabular}
}
\vspace{-1em}
\end{table}

\begin{table}[t]
\centering
\small
\caption{\textbf{Generalization of reasoning retention.}
Accuracy on five reasoning benchmarks before and after downstream fine-tuning.
\textbf{Post-RL} is the reference checkpoint before SFT. For each model, the preservation constraint is constructed
from a single anchor corpus: MATH-lighteval for Qwen2.5-3B and GSM8K for
Qwen2.5-1.5B. Performance on the remaining four benchmarks measures whether
reasoning preservation generalizes beyond the anchor corpus.}
\label{tab:reasoning_retention_two_models}
\setlength{\tabcolsep}{3pt}
\resizebox{\textwidth}{!}{%
\begin{tabular}{l c c c c c c c c}
\toprule
\rowcolor{white}
\multirow{2}{*}{\textbf{Benchmarks}}
& \multicolumn{1}{c}{\textbf{Before SFT}}
& \multicolumn{7}{c}{\textbf{After SFT}} \\
\cmidrule(lr){2-2}\cmidrule(lr){3-9}
\rowcolor{white}
& \textbf{Post-RL}
& \textbf{Vanilla LoRA}
& \textbf{LoRA-Null}
& \textbf{MiLoRA}
& \textbf{NSCL}
& \textbf{Exp. Replay}
& \textbf{SDFT}
& \cellcolor{row-highlight}\textbf{Ours} \\
\midrule
\rowcolor{white}
\multicolumn{9}{c}{\textbf{Qwen2.5-3B}} \\
\midrule
\rowcolor{gray!10}
MATH-lighteval   & 0.638 & 0.089 & 0.387 & 0.160 & \textbf{0.637} & 0.618 & 0.632 & \cellcolor{row-highlight}0.635 \\
MATH-500         & 0.624 & 0.052 & 0.428 & 0.150 & 0.622 & \textbf{0.626} & 0.608 & \cellcolor{row-highlight}0.622 \\
\rowcolor{gray!10}
GSM8K            & 0.842 & 0.000 & 0.034 & 0.007 & \textbf{0.844} & 0.562 & 0.784 & \cellcolor{row-highlight}0.843 \\
AGIEval-EN-MATH  & 0.620 & 0.051 & 0.371 & 0.156 & 0.611 & 0.603 & 0.607 & \cellcolor{row-highlight}\textbf{0.621} \\
\rowcolor{gray!10}
MMLU-STEM        & 0.659 & 0.420 & 0.487 & 0.528 & \textbf{0.636} & 0.476 & 0.611 & \cellcolor{row-highlight}\textbf{0.636} \\
\cmidrule(lr){1-9}
Average          & 0.677 & 0.122 & 0.341 & 0.200 & 0.670 & 0.577 & 0.648 & \cellcolor{row-highlight}\textbf{0.671} \\
\midrule
\rowcolor{white}
\multicolumn{9}{c}{\textbf{Qwen2.5-1.5B}} \\
\midrule
\rowcolor{gray!10}
MATH-lighteval   & 0.462 & 0.082 & 0.131 & 0.083 & 0.455 & 0.388 & 0.448 & \cellcolor{row-highlight} \textbf{0.467} \\
MATH-500         & 0.468 & 0.068 & 0.120 & 0.086 & 0.462 & 0.396 & 0.460 & \cellcolor{row-highlight}\textbf{0.468} \\
\rowcolor{gray!10}
GSM8K            & 0.801 & 0.162 & 0.435 & 0.406 & 0.798 & 0.775 & 0.793 & \cellcolor{row-highlight}\textbf{0.801} \\
AGIEval-EN-MATH  & 0.445 & 0.071 & 0.127 & 0.087 & 0.451 & 0.374 & 0.432 & \cellcolor{row-highlight}\textbf{0.453} \\
\rowcolor{gray!10}
MMLU-STEM        & 0.547 & 0.451 & 0.461 & 0.456 & \textbf{0.547} & 0.522 & 0.542 & \cellcolor{row-highlight}\textbf{0.547} \\
\cmidrule(lr){1-9}
Average          & 0.545 & 0.167 & 0.255 & 0.224 & 0.543 & 0.491 & 0.535 & \cellcolor{row-highlight}\textbf{0.547} \\
\bottomrule
\end{tabular}
}
\vspace{-1em}
\end{table}

\textbf{Baselines.} We compare our method against vanilla LoRA and the following baselines.
  \begin{itemize}[leftmargin=*, itemsep=2pt, topsep=2pt, parsep=0pt]
      \item LoRA-Null~\citep{tang2026put}: A null-space-based initialization baseline. For each target layer, it uses the null basis of the hidden state matrix to decompose the backbone weight into a frozen residual component and a null-space-projected component, then initializes the LoRA factors from the projected component.
      \item MiLoRA~\citep{wang2025milora}: Another SVD-based LoRA baseline that decomposes each pretrained weight into principal and minor components, freezes the principal component in the base model, and initializes the LoRA from the minor singular components.
      \item NSCL~\citep{wang2021training}: A null-space-based projection method that trains the full model, projecting each layer's full gradient onto the approximate null space of the previous task's hidden states before each update.
      \item Experience Replay~\citep{rolnick2019experience}: It employs standard LoRA and trains the LoRA modules on each task's SFT data mixed with a small replay buffer of reasoning traces.
      \item SDFT~\citep{shenfeld2026selfdistillation}: Conditions the model on downstream demonstrations to generate on-policy rollouts, which are subsequently used to fine-tune the model to reduce forgetting.
  \end{itemize}

\subsection{Main Results}\label{sec:main_results}
\textbf{Adaptation and reasoning retention.}
Table~\ref{tab:target_math_tradeoff_two_models} reports target-task accuracy alongside reasoning accuracy on the benchmark used for RL post-training.
Vanilla LoRA achieves strong accuracy on the target tasks, but the reasoning ability acquired through RL degrades sharply, dropping from $63.8\%$ to $8.3\%$ on Qwen2.5-3B and from $80.1\%$ to $9.4\%$ on Qwen2.5-1.5B.
LoRA-Null and MiLoRA mitigate this collapse, yet neither closes the gap. SDFT largely preserves reasoning performance but struggles to learn the target tasks, especially on Qwen2.5-1.5B.
NB-LoRA, by contrast, holds reasoning accuracy at essentially its pre-fine-tuning level on both models while matching Vanilla LoRA's target accuracy, and it does so uniformly: its per-task variation is among the smallest of these methods, whereas Experience Replay loses more than forty points on individual tasks. 
NSCL is the only baseline that achieves comparable reasoning retention without severely compromising target-task performance. However, NSCL updates the full model and projects the gradients at every optimization step, whereas NB-LoRA encodes the constraint directly in the adapter parameterization, requiring no gradient projection while keeping the backbone frozen. 

\textbf{Generalization of reasoning retention.}
Table~\ref{tab:reasoning_retention_two_models} examines whether reasoning
preservation generalizes beyond the anchor corpus used to construct the null
basis. Although the constraint is estimated solely from MATH-lighteval for
Qwen2.5-3B and GSM8K for Qwen2.5-1.5B, NB-LoRA preserves reasoning performance
across all five benchmarks, including the four unseen during constraint
construction. Its average accuracy remains close to the Post-RL reference:
$67.1\%$ versus $67.7\%$ on Qwen2.5-3B and $54.3\%$ versus $54.5\%$ on
Qwen2.5-1.5B. These results suggest that the null space estimated from one
reasoning corpus protects representations shared across reasoning benchmarks.
Among the baselines, only NSCL provides comparable cross-benchmark retention,
while NB-LoRA offers substantially greater efficiency, as discussed next.

\subsection{Efficiency Comparison}\label{sec:efficiency}
Among the compared methods, NSCL and NB-LoRA achieve similar adaptation performance and reasoning retention, but differ substantially in training efficiency.
\begin{wrapfigure}{r}{0.52\linewidth}
    \centering
    \vspace{-10pt}
    \includegraphics[width=\linewidth]{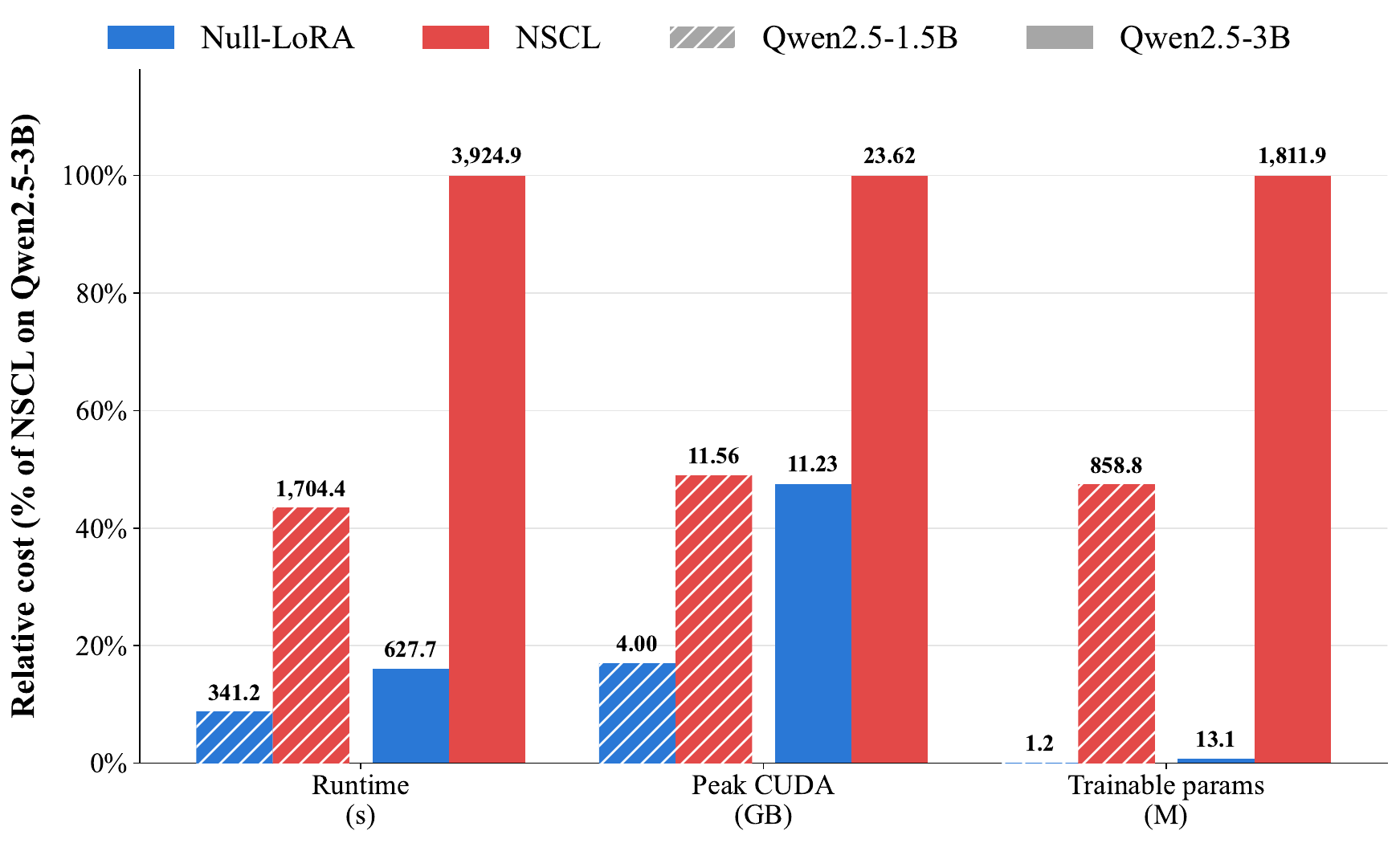}
    \caption{Training efficiency comparison between NB-LoRA and NSCL
    (3 epochs, batch size 1). Bars show runtime, peak CUDA memory, and trainable parameters relative to NSCL.} 
    \label{fig:efficiency_benchmark}
    \vspace{-10pt}
\end{wrapfigure}
As shown in~\Cref{fig:efficiency_benchmark}, NB-LoRA consistently reduces runtime, memory usage, and the number of trainable parameters across both model scales. On Qwen2.5-3B, it is $6.3\times$ faster and uses $2.1\times$ less peak CUDA memory than NSCL. The gap is structural: NSCL must keep full-model optimizer states together with a per-layer projector of size $O(d^2)$, while applying the projection incurs $O(d^3)$ cost per layer per optimizer step for hidden dimension $d$. In contrast, NB-LoRA enforces the null-space constraint directly through low-rank reparameterization, avoiding repeated projection during training. As model size increases, this structural difference leads to a widening efficiency gap, indicating better scalability of NB-LoRA to larger LLMs.


\subsection{Ablation Study}
\begin{figure}[t]
\centering
\begin{subfigure}[t]{0.32\linewidth}
\centering
\includegraphics[width=\linewidth]{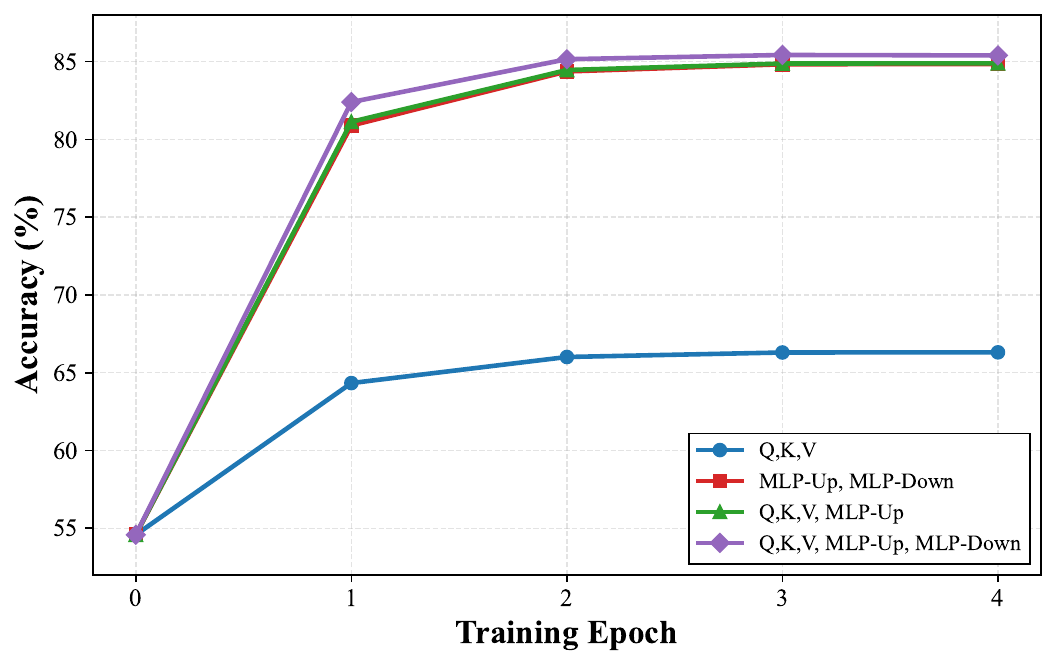}
\caption{Target-task convergence for different tunable modules on Qwen2.5-3B.}
\label{fig:ablation_module_selection}
\end{subfigure}
\hfill
\begin{subfigure}[t]{0.32\linewidth}
\centering
\includegraphics[width=\linewidth]{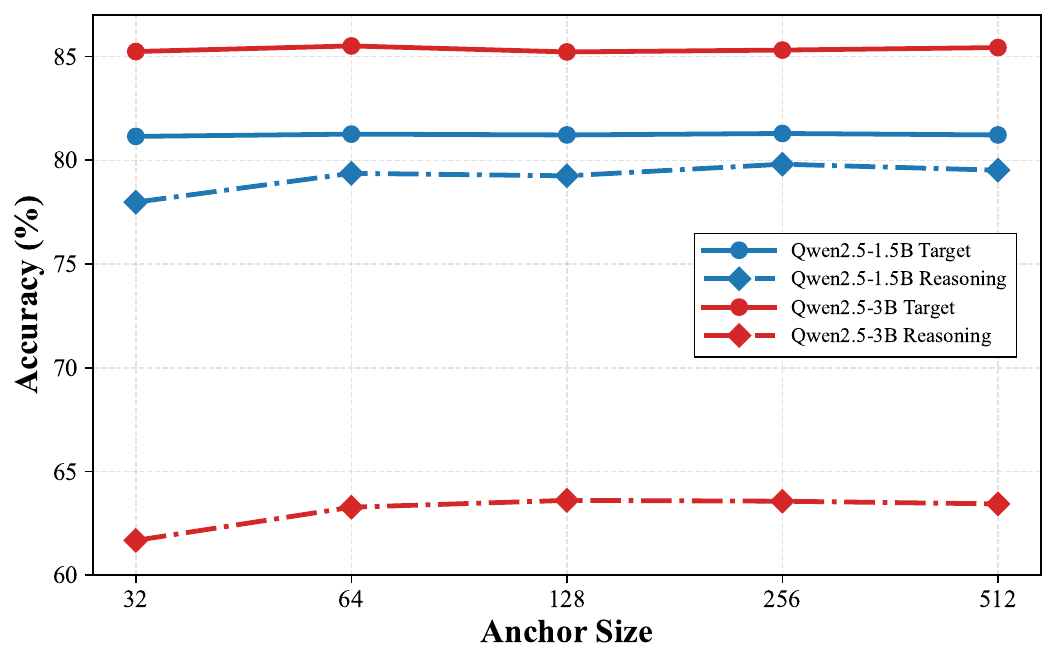}
\caption{Effect of the number of reasoning examples used to estimate the null basis.}
\label{fig:ablation_anchor_size}
\end{subfigure}
\hfill
\begin{subfigure}[t]{0.32\linewidth}
\centering
\includegraphics[width=\linewidth]{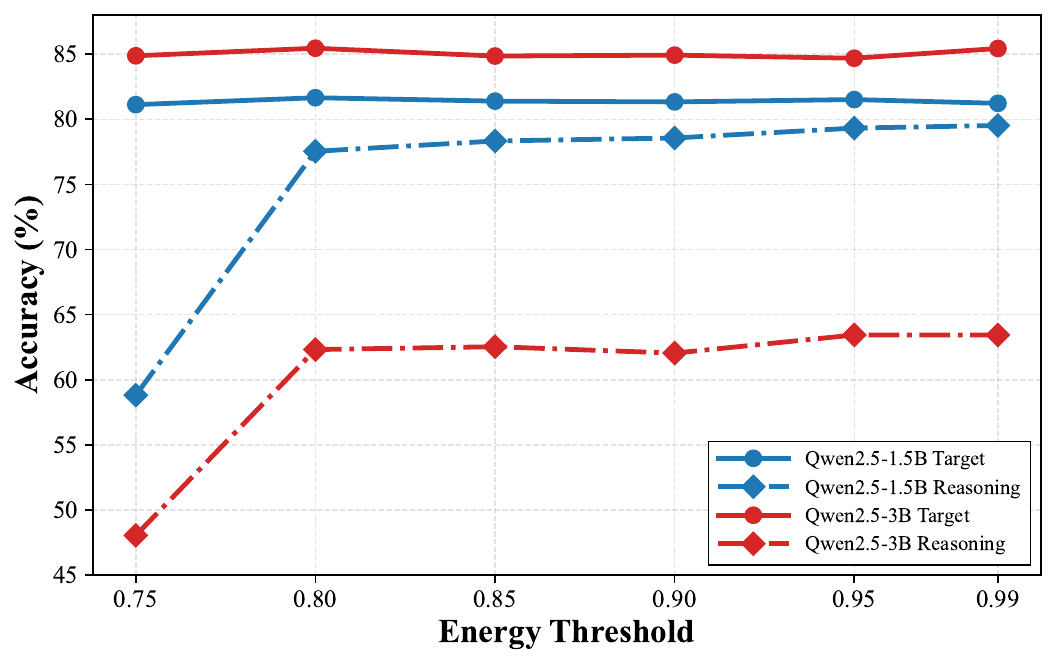}
\caption{Effect of the spectral-energy threshold used to define the approximate null space.}
\label{fig:ablation_energy_threshold}
\end{subfigure}
\caption{Ablations of the design choices of the proposed NB-LoRA. Target denotes average downstream accuracy over the eight commonsense reasoning adaptation tasks. Reason denotes average retained reasoning accuracy across the eight adapted checkpoints, evaluated on GSM8K for Qwen2.5-1.5B and MATH-lighteval for Qwen2.5-3B.}
\label{fig:ablation_overview}
\vspace{-1em}
\end{figure}

\textbf{Module selection.}
\Cref{fig:ablation_module_selection} compares four choices of tunable modules: only the attention Q/K/V projections, only the MLP up/down projections (i.e., the feed-forward module), Q/K/V plus the MLP up projection, and all five projections. Restricting NB-LoRA to Q/K/V substantially underfits the target tasks, whereas including MLP projections closes most of the gap. Tuning MLP up/down or Q/K/V+MLP-up achieves nearly the same target accuracy as tuning all five projections, which performs best overall. These results also reveal lower-cost alternatives to the full module set: Q/K/V+MLP-up avoids the higher-dimensional null basis required for MLP-down, while MLP-only tuning adapts fewer modules.

\textbf{Anchor data size.}
\Cref{fig:ablation_anchor_size} varies the number of reasoning examples used to estimate the null basis. Target-task accuracy remains nearly unchanged from 32 to 512 examples for both models, while reasoning retention is stable from 64 to 512 examples. Performance degrades mainly at 32 examples, particularly for Qwen2.5-3B, likely because the resulting null-space estimate is noisier. Together with the spectral analysis in Section~\ref{sec:sub:subspace_large_enough}, these results indicate that a moderate number of reasoning examples is sufficient to capture a stable null space, an important property for the practical use of NB-LoRA.

\textbf{Energy threshold.} 
\Cref{fig:ablation_energy_threshold} varies the spectral-energy threshold used to define the approximate null space. Across those thresholds, target-task accuracy remains stable for both models, which can be attributed to the null space being large enough to accommodate adaptation even under stricter preservation as we showed in Section~\ref{sec:sub:subspace_large_enough}.
Reasoning retention is robust for thresholds from 0.80 to 0.99, but drops sharply at 0.75. This suggests that NB-LoRA is insensitive to the threshold over a broad range, with degradation occurring only when the constraint becomes too permissive and more reasoning directions enter the trainable subspace. 
Such robustness provides flexibility in threshold selection, with higher thresholds favoring stronger preservation and moderately lower thresholds potentially allowing a larger trainable subspace when needed.


Detailed per-task results for the module-selection, anchor-size, and
energy-threshold ablations are provided in
\Cref{app:module_ablation_details,app:anchor_size_ablation_details,app:energy_threshold_ablation_details},
respectively.

\section{Conclusion}

We study post-RL adaptation: how to efficiently learn new tasks without erasing reasoning abilities acquired through RL. We find that standard LoRA can substantially degrade these abilities. Our analysis shows that reasoning activations concentrate in low-dimensional subspaces, leaving substantial null-space capacity for adaptation. This observation motivates Null-Basis LoRA (NB-LoRA), which constrains low-rank updates to an approximate null space of reasoning activations. The constraint is maintained throughout fine-tuning without replay or per-step gradient projection. We further show that NB-LoRA preserves the relevant smoothness constants and inherits standard LoRA's $O(1/\log(T))$ convergence guarantee under the same assumptions. Across two RL-trained models and diverse downstream tasks, NB-LoRA matches standard LoRA's adaptation performance while retaining reasoning accuracy near its pre-fine-tuning level. Its preservation generalizes to unseen reasoning benchmarks, and its training is more computationally and memory efficient than full-model null-space projection. Overall, our results demonstrate that NB-LoRA provides an efficient approach for adapting post-RL models to new tasks while preserving their reasoning abilities.

\bibliography{content/references,content/related_works_refs_to_check}
\bibliographystyle{iclr2025_conference}

\newpage
\appendix
\BeforeBeginEnvironment{figure}{\FloatBarrier}
\BeforeBeginEnvironment{table}{\FloatBarrier}
\clearpage

\begin{center}
    {\bf\Large Appendix}
\end{center}

\startcontents[sections]
\printcontents[sections]{l}{1}{\setcounter{tocdepth}{3}}

\clearpage

\section{Derivation of the SVD-Based Principal Angle Computation}\label{app:principal_angles}
For completeness, we derive the relationship between principal angles and the singular values of the cross-basis matrix~\citep{horn2012matrix,knyazev2012principal}.

Let $\mathcal{N}_1,\mathcal{N}_2 \subset \mathbb{R}^n$ be two subspaces with dimensions $k_1$ and $k_2$, where $k_1 \le k_2$, and let $\mathbf{N}_1 \in \mathbb{R}^{n \times k_1}$ and $\mathbf{N}_2 \in \mathbb{R}^{n \times k_2}$ be orthonormal bases of these subspaces. We show that the recursive definition of principal angles yields
\begin{equation}    
    \cos\theta_i=\sigma_i(\mathbf N_1^\top\mathbf N_2),\qquad i=1,\ldots,k_1.
\end{equation}

For the first principal angle, any vectors $\mathbf{x} \in \mathcal{N}_1$ and $\mathbf{y} \in \mathcal{N}_2$ can be written as
\begin{equation}
\mathbf{x} = \mathbf{N}_1 \mathbf{a}, \qquad \mathbf{y} = \mathbf{N}_2 \mathbf{b},
\end{equation}
for coefficient vectors $\mathbf{a} \in \mathbb{R}^{k_1}$ and $\mathbf{b} \in \mathbb{R}^{k_2}$. Since $\mathbf{N}_1$ and $\mathbf{N}_2$ have orthonormal columns, we have
\begin{equation}
\|\mathbf{x}\|_2 = \|\mathbf{a}\|_2, \qquad \|\mathbf{y}\|_2 = \|\mathbf{b}\|_2.
\end{equation}
Therefore, the first principal angle satisfies
\begin{equation}
\cos \theta_1
= \max_{\substack{\mathbf{x} \in \mathcal{N}_1,\ \mathbf{y} \in \mathcal{N}_2 \\
\|\mathbf{x}\|_2 = \|\mathbf{y}\|_2 = 1}}
\mathbf{x}^\top \mathbf{y}
= \max_{\substack{\mathbf{a} \in \mathbb{R}^{k_1},\ \mathbf{b} \in \mathbb{R}^{k_2} \\
\|\mathbf{a}\|_2 = \|\mathbf{b}\|_2 = 1}}
\mathbf{a}^\top \mathbf{N}_1^\top \mathbf{N}_2 \mathbf{b}.
\end{equation}
By the variational characterization of the singular values, this maximum is exactly the largest singular value of $\mathbf{N}_1^\top \mathbf{N}_2$. Hence,
\begin{equation}
\cos \theta_1 = \sigma_1(\mathbf{N}_1^\top \mathbf{N}_2).
\end{equation}

Now consider the singular value decomposition
\begin{equation}
\mathbf{N}_1^\top \mathbf{N}_2 = \mathbf{U}\mathbf{\Sigma}\mathbf{V}^\top,
\end{equation}
where $\mathbf{U} \in \mathbb{R}^{k_1 \times k_1}$ and $\mathbf{V} \in \mathbb{R}^{k_2 \times k_2}$ are orthogonal, and
\begin{equation}
\mathbf{\Sigma} =
\begin{bmatrix}
\operatorname{diag}(\sigma_1,\ldots,\sigma_{k_1}) & \mathbf{0}
\end{bmatrix}.
\end{equation}
Let $\mathbf{u}_i$ and $\mathbf{v}_i$ denote the left and right singular vectors associated with $\sigma_i$, and define
\begin{equation}
\mathbf{x}_i = \mathbf{N}_1 \mathbf{u}_i, \qquad \mathbf{y}_i = \mathbf{N}_2 \mathbf{v}_i.
\end{equation}
Because $\mathbf{N}_1$ and $\mathbf{N}_2$ have orthonormal columns and the singular vectors are orthonormal, the vectors $\mathbf{x}_1,\ldots,\mathbf{x}_{k_1}$ are orthonormal in $\mathcal{N}_1$, and $\mathbf{y}_1,\ldots,\mathbf{y}_{k_1}$ are orthonormal in $\mathcal{N}_2$. Moreover,
\begin{equation}
\mathbf{x}_i^\top \mathbf{y}_j
= \mathbf{u}_i^\top \mathbf{N}_1^\top \mathbf{N}_2 \mathbf{v}_j
= \mathbf{u}_i^\top \mathbf{U}\mathbf{\Sigma}\mathbf{V}^\top \mathbf{v}_j
= \sigma_j \mathbf{u}_i^\top \mathbf{u}_j.
\end{equation}
Thus, $\mathbf{x}_i^\top \mathbf{y}_i = \sigma_i$, and for $i \neq j$ these pairs satisfy the orthogonality constraints required by the recursive definition of principal angles. Applying the same maximization argument on the orthogonal complements after removing the first $i-1$ principal vector pairs yields
\begin{equation}
\cos \theta_i = \sigma_i, \qquad i = 1,\ldots,k_1.
\end{equation}
Therefore, the principal angles between $\mathcal{N}_1$ and $\mathcal{N}_2$ are given by
\begin{equation}
    \theta_i=\arccos \sigma_i(\mathbf N_1^\top\mathbf N_2),\qquad i=1,\ldots,k_1.
\end{equation}

\section{Convergence Analysis of the Proposed Algorithm}\label{app:convergence}
This appendix shows that the null-space reparameterization does not weaken the convergence guarantee for standard LoRA gradient descent established by \citet{mu2026convergence}. The key observation is that right-multiplication by $\mathbf{N}_s^{\top}$, where $\mathbf{N}_s$ has orthonormal columns, is an isometry under the Frobenius norm. Consequently, composing the training loss with this map preserves the smoothness and lower-bound assumptions required by their analysis. We formalize this observation in a bridging lemma and then apply their result to NB-LoRA.

Let $\mathbf{W} \in \mathbb{R}^{m \times n}$ be a frozen backbone weight and let $f: \mathbb{R}^{m \times n} \to \mathbb{R}$ denote the SFT loss on $\mathcal{D}_a$ as a function of its update $\Delta\mathbf{W}$, with the remaining model parameters fixed. Standard LoRA parameterizes this update as $\Delta\mathbf{W}=\mathbf{B}\mathbf{A}$, whereas NB-LoRA minimizes
\begin{equation}
g(\mathbf{B}, \mathbf{C}) := f\big(\mathbf{B}\mathbf{C}\mathbf{N}_s^{\top}\big)
\quad \text{over} \quad \mathbf{B} \in \mathbb{R}^{m \times r},\ \mathbf{C} \in \mathbb{R}^{r \times k},
\label{eq:app_null_lora_obj}
\end{equation}
where $\mathbf{N}_s \in \mathbb{R}^{n \times k}$ is the fixed null basis of \eqref{eq:deltaw}. Since $\mathbf{N}_s$ is an
orthonormal basis of the null space of $\mathbf{H}$, we have $\mathbf{N}_s^{\top}\mathbf{N}_s = \mathbf{I}_k$. We analyze the gradient descent updates
\begin{equation}
\mathbf{B}_{t+1} = \mathbf{B}_t - \eta_t \nabla_{\mathbf{B}}\, g(\mathbf{B}_t, \mathbf{C}_t),
\qquad
\mathbf{C}_{t+1} = \mathbf{C}_t - \eta_t \nabla_{\mathbf{C}}\, g(\mathbf{B}_t, \mathbf{C}_t),
\label{eq:app_gd_updates}
\end{equation}
where $\eta_t$ is the learning rate at step $t$.

Following standard convergence analyses of gradient descent methods~\citep{bottou2018optimization,bubeck2015convex,mu2026convergence}, we assume that the loss $f$ is differentiable and satisfies the standard smoothness and lower-boundedness conditions. Note that both are stated for the loss as a function of the \emph{full} weight update: no smoothness of the reparameterized objective $g$ and no boundedness of the trainable factors is assumed.

\begin{assumption}[$L$-smoothness]\label{ass:smooth}
The loss $f$ is differentiable, and there exists a constant $L \geq 1$ such that for all $\mathbf{X}, \mathbf{Y} \in \mathbb{R}^{m \times n}$,
\begin{equation}
\big\|\nabla f(\mathbf{X}) - \nabla f(\mathbf{Y})\big\|_F \leq L \big\|\mathbf{X} - \mathbf{Y}\big\|_F.
\end{equation}
\end{assumption}

\begin{assumption}[Lower boundedness]\label{ass:lower_bounded}
There exists $f^{*}$ such that $f(\mathbf{X}) \geq f^{*}$ for all $\mathbf{X} \in \mathbb{R}^{m \times n}$.
\end{assumption}

The following lemma bridges the two settings by showing that composition with the fixed basis $\mathbf{N}_s$ preserves both assumptions with the same constants.

\begin{lemma}[Null-basis composition preserves smoothness]\label{lem:smoothness_bridge}
Define $\tilde{f}: \mathbb{R}^{m \times k} \to \mathbb{R}$ by $\tilde{f}(\mathbf{M}) := f\big(\mathbf{M}\mathbf{N}_s^{\top}\big)$. Under \Cref{ass:smooth,ass:lower_bounded}, $\tilde{f}$ is $L$-smooth with the same constant $L$ and is lower bounded by the same constant $f^{*}$.
\end{lemma}

\begin{proof}
By the chain rule, $\nabla \tilde{f}(\mathbf{M}) = \nabla f\big(\mathbf{M}\mathbf{N}_s^{\top}\big)\mathbf{N}_s$. Since $\mathbf{N}_s^{\top}\mathbf{N}_s = \mathbf{I}_k$, right-multiplication by $\mathbf{N}_s^{\top}$ is a Frobenius isometry, $\|\mathbf{X}\mathbf{N}_s^{\top}\|_F^2 = \operatorname{Tr}(\mathbf{X}\mathbf{N}_s^{\top}\mathbf{N}_s\mathbf{X}^{\top}) = \|\mathbf{X}\|_F^2$ for all $\mathbf{X} \in \mathbb{R}^{m \times k}$, and the spectral norm of $\mathbf{N}_s$ equals one. Hence, for any $\mathbf{M}_1, \mathbf{M}_2 \in \mathbb{R}^{m \times k}$,
\begin{align}
\big\|\nabla \tilde{f}(\mathbf{M}_1) - \nabla \tilde{f}(\mathbf{M}_2)\big\|_F
&= \Big\|\big[\nabla f\big(\mathbf{M}_1\mathbf{N}_s^{\top}\big) - \nabla f\big(\mathbf{M}_2\mathbf{N}_s^{\top}\big)\big]\mathbf{N}_s\Big\|_F \nonumber\\
&\leq \big\|\nabla f\big(\mathbf{M}_1\mathbf{N}_s^{\top}\big) - \nabla f\big(\mathbf{M}_2\mathbf{N}_s^{\top}\big)\big\|_F \nonumber\\
&\leq L\big\|(\mathbf{M}_1 - \mathbf{M}_2)\mathbf{N}_s^{\top}\big\|_F
= L\big\|\mathbf{M}_1 - \mathbf{M}_2\big\|_F,
\end{align}
where the first inequality uses $\|\mathbf{Y}\mathbf{N}_s\|_F \leq \|\mathbf{Y}\|_F\,\|\mathbf{N}_s\|_2 = \|\mathbf{Y}\|_F$, the second uses \Cref{ass:smooth}, and the final equality uses the isometry property. Lower boundedness is immediate: $\tilde{f}(\mathbf{M}) = f(\mathbf{M}\mathbf{N}_s^{\top}) \geq f^{*}$.
\end{proof}

By \eqref{eq:app_null_lora_obj}, $g(\mathbf{B}, \mathbf{C}) = \tilde{f}(\mathbf{B}\mathbf{C})$, so the updates \eqref{eq:app_gd_updates} are precisely the original LoRA gradient descent algorithm applied to the loss $\tilde{f}$, with input dimension $k$ in place of $n$. By \Cref{lem:smoothness_bridge}, $\tilde{f}$ satisfies the assumptions of \citet{mu2026convergence} with unchanged constants, their main convergence result thus applies directly to NB-LoRA.

\begin{theorem}[Convergence of NB-LoRA; Theorem~3.5 of \citealp{mu2026convergence}]\label{thm:null_lora_convergence}
Suppose \Cref{ass:smooth,ass:lower_bounded} hold. Run $T$ steps of the updates \eqref{eq:app_gd_updates} with the learning rate
\begin{equation}
\eta_t = \min\left\{\frac{1}{4\sqrt{2}\,L\big(\|\mathbf{V}_t\|_F^2 + \|\nabla \tilde{f}(\mathbf{B}_t\mathbf{C}_t)\|_F\big)},\ 1\right\},
\qquad
\mathbf{V}_t = \begin{bmatrix}\mathbf{B}_t \\ \mathbf{C}_t^{\top}\end{bmatrix} \in \mathbb{R}^{(m+k) \times r}.
\label{eq:app_stepsize}
\end{equation}
Then
\begin{equation}
\min_{t = 0, \ldots, T-1}\ \big\|\nabla g(\mathbf{B}_t, \mathbf{C}_t)\big\|_F^2 = O\!\left(\frac{1}{\log T}\right),
\end{equation}
where the $O(\cdot)$ notation hides dependence on $\mathbf{V}_0$, $g(\mathbf{B}_0, \mathbf{C}_0)$, $f^{*}$, and $L$. Moreover, if there exists a constant $C_V > 0$ such that $\|\mathbf{V}_t\|_F \leq C_V$ for all $t$, the rate improves to $O(1/T)$.
\end{theorem}

\begin{remark}[Same guarantee as standard LoRA]
\Cref{lem:smoothness_bridge} shows that the smoothness constant is preserved \emph{exactly}: right-composition with $\mathbf{N}_s^{\top}$ is an isometric embedding of $\mathbb{R}^{m \times k}$ into $\mathbb{R}^{m \times n}$. Consequently, NB-LoRA satisfies the same assumptions and inherits the same convergence rate as standard LoRA.
Overall, the null-space reparameterization reduces the input dimension from $n$ to $k$ without weakening the convergence guarantee.
\end{remark}

\section{Further Experiments}

\subsection{Detailed Ablation Results}\label{app:ablation_details}

This appendix provides detailed per-task results for the ablation studies in~\Cref{fig:ablation_overview}. Target denotes downstream-task accuracy at the post-SFT checkpoint, while Reason denotes reasoning accuracy on GSM8K for Qwen2.5-1.5B and MATH-lighteval for Qwen2.5-3B. All values are percentages.

\subsubsection{Per-Task Results of the Module-Selection Ablation}\label{app:module_ablation_details}
Table~\ref{tab:module_ablation_details} reports the per-task results underlying the module-selection ablation in \Cref{fig:ablation_module_selection}. The detailed breakdown reinforces the main-text observation that target-task adaptation relies strongly on the MLP projections. Tuning only Q/K/V underfits severely on several tasks (e.g., on Qwen2.5-3B, ARC-Easy drops from $93.69\%$ under the full module set to $44.87\%$, and OpenBookQA from $88.80\%$ to $46.80\%$), whereas configurations including MLP projections recover most of the performance of the full module set. In particular, MLP-only and Q/K/V+MLP-Up achieve average target accuracy within one point of the full configuration for both models, supporting their use as lower-cost alternatives when reducing the number or dimensionality of adapted modules is desirable. Meanwhile, reasoning accuracy remains stable across the four module choices, indicating that the preservation behavior of NB-LoRA is largely insensitive to which modules are adapted.

\begin{table}[t]
\centering
\scriptsize
\caption{\textbf{Per-task results of the module-selection ablation in~\Cref{fig:ablation_module_selection}.} Each column group corresponds to one choice of tunable modules: Q, K, and V are the query, key, and value projections of the attention layers, and MLP-Up/MLP-Down are the up- and down-projections of the feed-forward block. Target and Reason denote target-task and reasoning accuracy (\%) at the post-SFT checkpoint.}
\label{tab:module_ablation_details}
\setlength{\tabcolsep}{3pt}
\resizebox{\textwidth}{!}{%
\begin{tabular}{l cc cc cc cc}
\toprule
\multirow{2}{*}{\textbf{Task}}
& \multicolumn{2}{c}{\textbf{Q,K,V}}
& \multicolumn{2}{c}{\textbf{Q,K,V, MLP-Up}}
& \multicolumn{2}{c}{\textbf{MLP-Up, MLP-Down}}
& \multicolumn{2}{c}{\textbf{Q,K,V, MLP-Up, MLP-Down}} \\
\cmidrule(lr){2-3}\cmidrule(lr){4-5}\cmidrule(lr){6-7}\cmidrule(lr){8-9}
& \textbf{Target} & \textbf{Reason}
& \textbf{Target} & \textbf{Reason}
& \textbf{Target} & \textbf{Reason}
& \textbf{Target} & \textbf{Reason} \\
\midrule
\multicolumn{9}{c}{\textbf{Qwen2.5-1.5B} (Reasoning task: GSM8K)} \\
\midrule
ARC-Challenge & 45.90 & 79.91 & \textcolor{black}{72.35} & 79.98 & \textcolor{black}{74.74} & 80.14 & \textbf{77.05} & 79.91 \\
ARC-Easy      & 79.76 & 79.30 & 89.02 & 79.91 & 90.45 & 78.70 & \textbf{91.20} & 79.98 \\
BoolQ         & 65.60 & 79.61 & 69.60 & 79.98 & \textbf{69.88} & 79.98 & 68.31 & 79.53 \\
HellaSwag     & 90.26 & 79.61 & \textbf{91.58} & 78.62 & 91.23 & 79.45 & 91.37 & 79.83 \\
OpenBookQA    & 68.60 & 79.68 & 84.20 & 79.23 & \textbf{85.00} & 79.53 & 84.60 & 78.62 \\
PIQA          & 82.59 & 78.77 & 83.41 & 79.76 & 83.46 & 79.23 & \textbf{84.30} & 79.53 \\
SocialIQA     & \textbf{77.43} & 79.61 & 77.28 & 78.77 & 76.87 & 78.70 & 76.73 & 78.77 \\
WinoGrande    & 73.40 & 79.53 & \textbf{78.22} & 78.70 & 75.37 & 79.38 & 76.19 & 79.98 \\
\midrule
Average       & 72.94 & 79.50 & 80.71 & 79.37 & 80.88 & 79.39 & \textbf{81.22} & 79.52 \\
\midrule
\multicolumn{9}{c}{\textbf{Qwen2.5-3B} (Reasoning task: MATH-lighteval)} \\
\midrule
ARC-Challenge & 54.35 & 63.28 & \textcolor{black}{79.22} & 63.32 & \textcolor{black}{81.83} & 63.64 & \textbf{82.56} & 63.50 \\
ARC-Easy      & 44.87 & 63.28 & 92.34 & 63.74 & 93.39 & 63.60 & \textbf{93.69} & 63.16 \\
BoolQ         & 66.06 & 63.80 & 70.18 & 63.68 & 71.04 & 63.34 & \textbf{71.28} & 63.40 \\
HellaSwag     & 93.24 & 63.00 & \textbf{94.78} & 63.56 & 94.60 & 63.10 & 94.24 & 63.54 \\
OpenBookQA    & 46.80 & 63.46 & \textbf{89.00} & 63.38 & 87.40 & 62.90 & 88.80 & 63.88 \\
PIQA          & 76.28 & 63.18 & 87.21 & 63.04 & \textbf{87.65} & 63.36 & 87.16 & 63.60 \\
SocialIQA     & 67.96 & 63.70 & \textbf{81.78} & 64.36 & 79.99 & 63.02 & 81.11 & 63.32 \\
WinoGrande    & 80.82 & 63.44 & 84.53 & 63.26 & 82.72 & 63.34 & \textbf{84.61} & 63.10 \\
\midrule
Average       & 66.30 & 63.39 & 84.88 & 63.54 & 84.83 & 63.29 & \textbf{85.43} & 63.44 \\
\bottomrule
\end{tabular}%
}
\end{table}

\subsubsection{Per-Task Results of the Anchor-Size Ablation}\label{app:anchor_size_ablation_details}
Table~\ref{tab:anchor_size_ablation_details} reports the per-task results underlying the anchor-size ablation in \Cref{fig:ablation_anchor_size}, where the number of reasoning examples used to estimate the hidden-state null basis varies from \textcolor{black}{32} to 512. 
Consistent with Figure 5b, target-task accuracy is largely insensitive to anchor size: the average target accuracy varies by at most 0.3 points for both models, and no individual task exhibits a systematic trend. Reasoning retention is similarly stable once the anchor set contains at least 64 examples, with reasoning accuracy remaining close to the pre-fine-tuning level across adapted checkpoints. The clearest degradation occurs at the smallest anchor set of 32 examples: Qwen2.5-1.5B loses roughly one to two points of reasoning accuracy on most tasks, while Qwen2.5-3B drops sharply on WinoGrande to $49.20\%$, likely reflecting a noisier estimate of the reasoning subspace from too few examples. These per-task results further support the main-text conclusion that a moderate anchor set, around 64 examples in our experiments, is sufficient to obtain a stable null-space estimate, making NB-LoRA practical even with limited anchor data.

\begin{table}[t]
\centering
\scriptsize
\caption{\textbf{Per-task results of the anchor-size ablation in \Cref{fig:ablation_anchor_size}.} Each column group corresponds to the number of reasoning examples used to estimate the null basis. Target and Reason denote target-task and reasoning accuracy (\%) at the post-SFT checkpoint. Target accuracy remains nearly unchanged across anchor sizes, while reasoning retention is stable for 64 or more examples and degrades mainly at the smallest anchor set.}
\label{tab:anchor_size_ablation_details}
\setlength{\tabcolsep}{2.5pt}
\resizebox{\textwidth}{!}{%
\begin{tabular}{l cc cc cc cc cc}
\toprule
\multirow{2}{*}{\textbf{Task}}
& \multicolumn{2}{c}{\textbf{32}}
& \multicolumn{2}{c}{\textbf{64}}
& \multicolumn{2}{c}{\textbf{128}}
& \multicolumn{2}{c}{\textbf{256}}
& \multicolumn{2}{c}{\textbf{512}} \\
\cmidrule(lr){2-3}\cmidrule(lr){4-5}\cmidrule(lr){6-7}\cmidrule(lr){8-9}\cmidrule(lr){10-11}
& \textbf{Target} & \textbf{Reason}
& \textbf{Target} & \textbf{Reason}
& \textbf{Target} & \textbf{Reason}
& \textbf{Target} & \textbf{Reason}
& \textbf{Target} & \textbf{Reason} \\
\midrule
\multicolumn{11}{c}{\textbf{Qwen2.5-1.5B} (Reasoning task: GSM8K)} \\
\midrule
ARC-Challenge & 75.60 & \textcolor{black}{77.68} & 75.43 & 79.61 & 76.45 & 79.83 & \textcolor{black}{76.93} & 80.97 & 77.05 & 79.91 \\
ARC-Easy      & 90.49 & \textcolor{black}{78.21} & 90.45 & 79.68 & 90.82 & 79.76 & 89.98 & 79.76 & 91.20 & 79.98 \\
BoolQ         & 68.62 & \textcolor{black}{77.00} & 69.24 & 79.08 & 69.69 & 79.45 & 68.29 & 79.30 & 68.31 & 79.53 \\
HellaSwag     & 91.26 & \textcolor{black}{78.32} & 91.47 & 79.08 & 91.10 & 78.39 & 91.22 & 80.67 & 91.37 & 79.83 \\
OpenBookQA    & 85.20 & \textcolor{black}{78.29} & 84.60 & 79.98 & 84.00 & 79.76 & 86.80 & 79.61 & 84.60 & 78.62 \\
PIQA          & 83.03 & \textcolor{black}{77.61} & 84.77 & 78.92 & 84.06 & 78.77 & 84.11 & 79.68 & 84.30 & 79.53 \\
SocialIQA     & 76.71 & 78.85 & 76.87 & 79.98 & 77.33 & 78.77 & 76.92 & 79.98 & 76.73 & 78.77 \\
WinoGrande    & 78.30 & 77.71 & 77.27 & 78.62 & 76.32 & 79.23 & 76.09 & 78.47 & 76.19 & 79.98 \\
\midrule
Average       & 81.15 & 77.96 & 81.26 & 79.37 & 81.22 & 79.25 & 81.29 & 79.81 & 81.22 & 79.52 \\
\midrule
\multicolumn{11}{c}{\textbf{Qwen2.5-3B} (Reasoning task: MATH-lighteval)} \\
\midrule
ARC-Challenge & 82.85 & 63.84 & 82.60 & 63.82 & 82.60 & 63.52 & 82.53 & 63.56 & 82.56 & 63.50 \\
ARC-Easy      & 93.73 & 63.72 & 93.52 & 63.38 & 94.28 & 63.18 & \textcolor{black}{93.50} & 64.26 & 93.69 & 63.16 \\
BoolQ         & 71.07 & 63.72 & \textcolor{black}{71.40} & 63.78 & \textcolor{black}{71.39} & 63.96 & 71.41 & 63.70 & 71.28 & 63.40 \\
HellaSwag     & 94.04 & 63.96 & 94.65 & 63.12 & 94.46 & 63.58 & 94.29 & 63.76 & 94.24 & 63.54 \\
OpenBookQA    & 88.80 & 63.20 & 88.20 & 63.52 & 87.80 & 63.96 & 88.40 & 63.34 & 88.80 & 63.88 \\
PIQA          & 86.94 & 62.82 & 87.38 & 63.30 & 87.76 & 63.34 & 86.89 & 63.64 & 87.16 & 63.60 \\
SocialIQA     & \textcolor{black}{81.32} & 62.94 & \textcolor{black}{81.38} & 63.16 & 80.40 & 63.64 & 81.32 & 63.16 & 81.11 & 63.32 \\
WinoGrande    & \textcolor{black}{83.19} & \textcolor{black}{49.20} & 84.93 & 62.14 & 83.03 & 63.68 & \textcolor{black}{84.11} & 63.14 & 84.61 & 63.10 \\
\midrule
Average       & 85.24 & 61.68 & 85.51 & 63.28 & 85.22 & 63.61 & 85.31 & 63.57 & 85.43 & 63.44 \\
\bottomrule
\end{tabular}%
}
\end{table}

\subsubsection{Per-Task Results of the Energy-Threshold Ablation}\label{app:energy_threshold_ablation_details}
Table~\ref{tab:energy_threshold_ablation_details} reports the per-task results underlying the energy-threshold ablation in \Cref{fig:ablation_energy_threshold}, where the spectral-energy threshold defining the approximate null space varies from 0.75 to 0.99. The breakdown mirrors the findings in \Cref{fig:ablation_energy_threshold}. Target-task accuracy is stable across the entire range: raising the threshold, and thus shrinking the trainable subspace, does not degrade adaptation on any task, in line with the spectral analysis in Section~\ref{sec:sub:subspace_large_enough}, which shows that the null space remains large even at a 0.99 energy threshold. 
Reasoning retention is robust for thresholds from 0.80 to 0.99 but collapses at 0.75, and the per-task view reveals that this collapse is task-dependent rather than uniform: on Qwen2.5-1.5B, adaptation on HellaSwag, PIQA, and SocialIQA reduces reasoning accuracy to $31.54\%$, $40.64\%$, and $28.89\%$, respectively, whereas ARC and OpenBookQA leave it largely intact. The same tasks produce the largest degradations on Qwen2.5-3B, including $19.12\%$ after HellaSwag and $10.92\%$ after SocialIQA. 
This pattern is consistent with a permissive threshold admitting more reasoning-relevant directions into the trainable subspace, with the resulting degradation depending on the target task being adapted. 
Finally, within the robust range, retention on Qwen2.5-1.5B improves mildly as the threshold increases (from $77.54\%$ at 0.80 to $79.52\%$ at 0.99 on average). This wide robust range suggests some flexibility in threshold selection: if a target task may benefit from a larger update subspace, moderately lowering the threshold (e.g., to 0.95 or 0.90) could provide additional adaptation capacity with only a marginal reduction in reasoning retention.

\begin{table}[t]
\centering
\scriptsize
\caption{\textbf{Per-task results of the energy-threshold ablation in ~\Cref{fig:ablation_energy_threshold}.} Each column group corresponds to one spectral-energy threshold used to define the approximate null space. Target and Reason denote target-task and reasoning accuracy (\%) at the post-SFT checkpoint. Target accuracy remains stable across thresholds, while reasoning retention is robust from 0.80 to 0.99 and degrades sharply at 0.75 for checkpoints adapted to several target tasks.}
\label{tab:energy_threshold_ablation_details}
\setlength{\tabcolsep}{2pt}
\resizebox{\textwidth}{!}{%
\begin{tabular}{l cc cc cc cc cc cc}
\toprule
\multirow{2}{*}{\textbf{Task}}
& \multicolumn{2}{c}{\textbf{0.75}}
& \multicolumn{2}{c}{\textbf{0.80}}
& \multicolumn{2}{c}{\textbf{0.85}}
& \multicolumn{2}{c}{\textbf{0.90}}
& \multicolumn{2}{c}{\textbf{0.95}}
& \multicolumn{2}{c}{\textbf{0.99}} \\
\cmidrule(lr){2-3}\cmidrule(lr){4-5}\cmidrule(lr){6-7}\cmidrule(lr){8-9}\cmidrule(lr){10-11}\cmidrule(lr){12-13}
& \textbf{Target} & \textbf{Reason}
& \textbf{Target} & \textbf{Reason}
& \textbf{Target} & \textbf{Reason}
& \textbf{Target} & \textbf{Reason}
& \textbf{Target} & \textbf{Reason}
& \textbf{Target} & \textbf{Reason} \\
\midrule
\multicolumn{13}{c}{\textbf{Qwen2.5-1.5B} (Reasoning task: GSM8K)} \\
\midrule
ARC-Challenge & 77.13 & 78.54 & 78.07 & 80.36 & 77.99 & 80.44 & 75.60 & 80.14 & 76.54 & 79.98 & 77.05 & 79.91 \\
ARC-Easy      & 90.24 & 75.28 & 91.04 & 79.76 & 91.04 & 79.53 & 91.04 & 80.29 & 90.99 & 79.53 & 91.20 & 79.98 \\
BoolQ         & 68.10 & 61.03 & 69.33 & 70.89 & 68.72 & 75.51 & 68.69 & 72.63 & 70.09 & 78.47 & 68.31 & 79.53 \\
HellaSwag     & 90.85 & 31.54 & 90.98 & 74.68 & 90.68 & 75.51 & 91.04 & 79.45 & 90.97 & 79.45 & 91.37 & 79.83 \\
OpenBookQA    & 84.00 & 77.33 & 86.20 & 79.68 & 85.20 & 79.38 & 85.60 & 80.36 & 84.60 & 79.91 & 84.60 & 78.62 \\
PIQA          & 85.15 & 40.64 & 84.06 & 78.70 & 83.62 & 78.85 & 83.68 & 78.70 & 84.71 & 78.54 & 84.30 & 79.53 \\
SocialIQA     & 76.92 & 28.89 & 76.31 & 79.15 & 76.61 & 78.32 & 76.87 & 78.47 & 75.90 & 79.38 & 76.73 & 78.77 \\
WinoGrande    & 76.56 & 77.33 & 77.19 & 77.10 & 77.27 & 79.08 & 78.14 & 78.47 & 78.30 & 79.30 & 76.19 & 79.98 \\
\midrule
Average       & 81.12 & 58.82 & 81.65 & 77.54 & 81.39 & 78.33 & 81.33 & 78.56 & 81.51 & 79.32 & 81.22 & 79.52 \\
\midrule
\multicolumn{13}{c}{\textbf{Qwen2.5-3B} (Reasoning task: MATH-lighteval)} \\
\midrule
ARC-Challenge & 83.02 & 63.06 & 82.51 & 63.14 & 83.62 & 63.32 & 82.42 & 63.20 & \textcolor{black}{80.38} & 63.58 & 82.56 & 63.50 \\
ARC-Easy      & 93.94 & 58.80 & 94.15 & 62.98 & 92.85 & 63.76 & 93.90 & 63.86 & 93.01 & 63.62 & 93.69 & 63.16 \\
BoolQ         & 71.01 & 58.68 & 72.14 & 61.74 & 71.68 & 62.08 & 71.41 & 62.54 & 71.19 & 64.22 & 71.28 & 63.40 \\
HellaSwag     & 92.10 & 19.12 & 94.03 & 62.26 & 93.29 & 57.28 & 93.42 & 53.68 & 93.88 & 61.84 & 94.24 & 63.54 \\
OpenBookQA    & 88.00 & 55.56 & 89.00 & 61.28 & 87.20 & 63.16 & 87.60 & 63.66 & 88.00 & 63.62 & 88.80 & 63.88 \\
\textcolor{black}{PIQA}          & 87.00 & 54.90 & 86.06 & 61.76  & 86.34 & 63.38 & 86.07 & 63.02 & 86.34 & 63.58 & 87.16 & 63.60 \\
\textcolor{black}{SocialIQA}     & 81.27 & 10.92 & 81.17 & 62.04 & 81.10 & 63.58 & 81.05 & 63.16 & 81.07 & 63.82 & 81.11 & 63.32 \\
WinoGrande    & 82.64 & 63.42 & 84.53 & 63.38 & 82.72 & 63.76 & 83.50 & 63.28 & 83.58 & 63.24 & 84.61 & 63.10 \\
\midrule
Average       & 84.87 & 48.06 & 85.45 & 62.32 & 84.85 & 62.54 & 84.92 & 62.05 & 84.68 & 63.44 & 85.43 & 63.44 \\
\bottomrule
\end{tabular}%
}
\end{table}

\subsection{Null-Space Capacity for Sequential Multi-Task Adaptation}
\label{app:appendix_sequential_adaptation}

The experiments in \Cref{sec:main_results} demonstrate that NB-LoRA preserves reasoning during adaptation to an individual downstream task. We next investigate whether the null space provides sufficient capacity for multiple downstream tasks. Starting from each post-RL checkpoint, we first jointly adapt the model to eight commonsense tasks using NB-LoRA. We then merge the learned update, recompute the null basis using examples representing both the original reasoning capability and the newly acquired commonsense capabilities, and sequentially adapt the resulting model to tool use. This setting tests whether the null space can support cumulative capability acquisition rather than only a single adaptation.

\begin{table}[H]
    \centering
    \caption{\textbf{Accuracy across sequential adaptation stages.} Reasoning, commonsense, and tool-use accuracy at the original reasoning-tuned checkpoint, after eight-task commonsense fine-tuning, and after sequential tool-use fine-tuning. Both adaptation stages use NB-LoRA. FT denotes fine-tuning.}
    \label{tab:sequential_adaptation_accuracy}
    \small
    \setlength{\tabcolsep}{4pt}
    \begin{tabularx}{\textwidth}{@{}lXccc@{}}
        \toprule
        Model & Stage & Reasoning acc. & Commonsense acc. & Tool-use acc. \\
        \midrule
        \multirow{3}{*}{Qwen2.5-1.5B GSM8K}
            & Prior to commonsense FT & $79.30\%$ & $33.30\%$ & $14.52\%$ \\
            & After commonsense FT    & $79.61\%$ & $79.52\%$ & $6.45\%$  \\
            & After tool-use FT       & $79.83\%$ & $79.72\%$ & $72.58\%$ \\
        \midrule
        \multirow{3}{*}{Qwen2.5-3B MATH}
            & Prior to commonsense FT & $63.64\%$ & $54.58\%$ & $11.29\%$ \\
            & After commonsense FT    & $63.78\%$ & $84.62\%$ & $35.48\%$ \\
            & After tool-use FT       & $63.86\%$ & $84.68\%$ & $74.19\%$ \\
        \bottomrule
    \end{tabularx}
\end{table}

As shown in \Cref{tab:sequential_adaptation_accuracy}, the models successfully accumulate new capabilities across the two adaptation stages. Commonsense accuracy increases from $33.30\%$ to $79.52\%$ for Qwen2.5-1.5B and from $54.58\%$ to $84.62\%$ for Qwen2.5-3B. Subsequent tool-use adaptation raises tool-use accuracy to $72.58\%$ and $74.19\%$, respectively, while preserving the acquired commonsense performance. Reasoning accuracy also remains near its post-RL level throughout the sequence. These results demonstrate that the approximate null space can accommodate multiple heterogeneous downstream tasks while preserving both the original reasoning ability and capabilities acquired during earlier adaptation stages.

\begin{figure}[H]
    \centering
    \includegraphics[width=\linewidth]{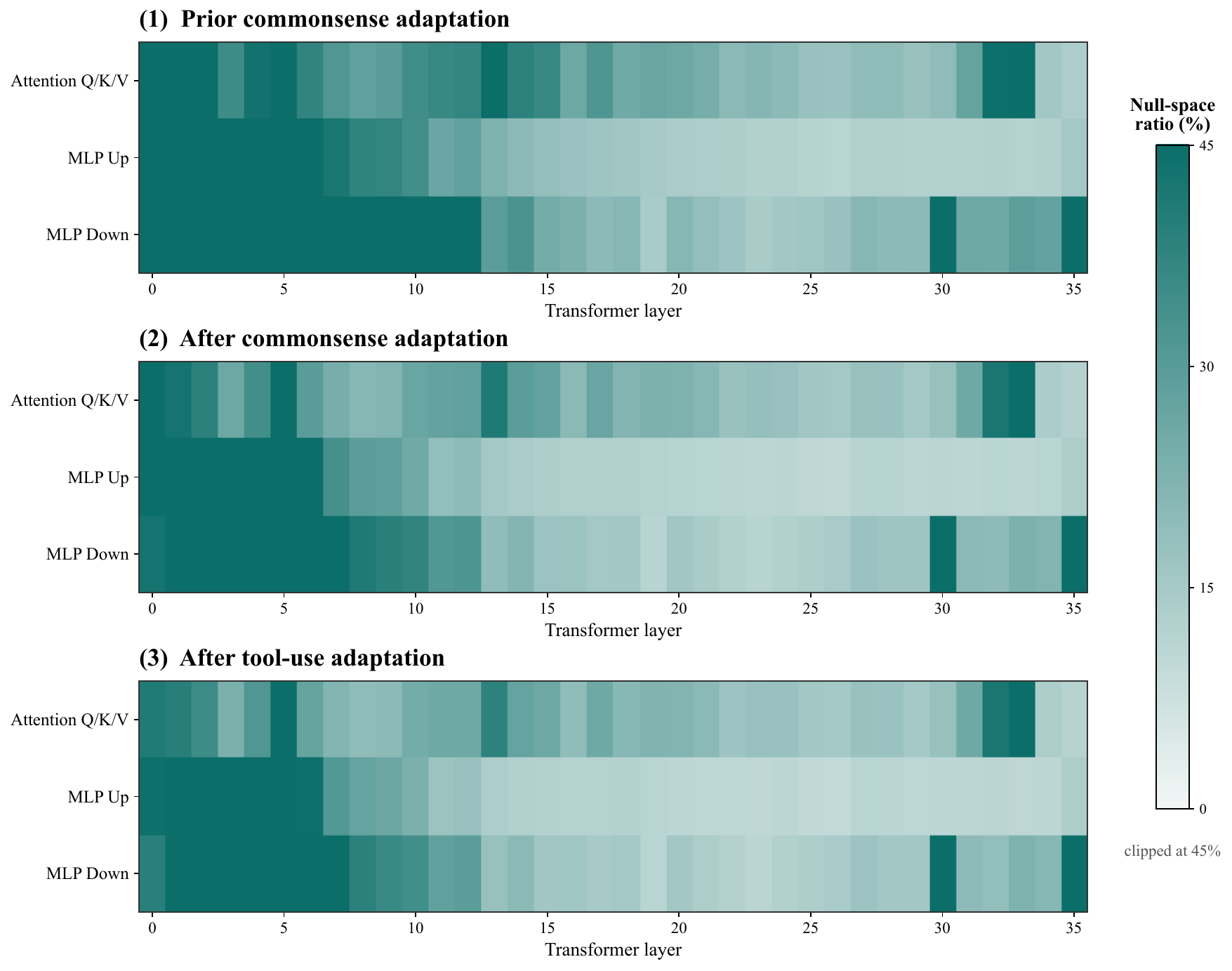}
    \caption[Layer-wise null-space ratios show that sequential adaptation reduces but does not eliminate low-energy capacity; remaining capacity is broadly distributed, with a bottleneck in middle MLP down-projection layers.]{\textbf{Layer-wise null-space capacity across adaptation stages.} Heatmaps show the null-space ratio for attention projections, MLP up-projections, and MLP down-projections across transformer layers at three stages of sequential adaptation. All panels share the same color scale, clipped at 45\% for visibility. Although the retained null space contracts across the adaptation stages, it remains broadly distributed across layers.}
    \label{fig:sequential_null_space_heatmap}
\end{figure}
As the protected capability set expands, its activation subspace occupies more directions and the corresponding approximate null space gradually contracts (see \Cref{fig:sequential_null_space_capacity}). Nevertheless, after adaptation to the eight commonsense tasks and tool use, the Qwen2.5-3B model retains approximately \textbf{$25.0\%$} null-space capacity in attention projections, \textbf{$23.0\%$} in MLP up-projections, and \textbf{$33.4\%$} in MLP down-projections. The layer-wise heatmaps in \Cref{fig:sequential_null_space_heatmap} further show that the remaining capacity is broadly distributed across the network. Together, these results indicate that sequential adaptation with NB-LoRA does not exhaust the available null space, leaving a substantial subspace for further adaptation.

\begin{figure}[H]
    \centering
    \includegraphics[width=0.7\linewidth]{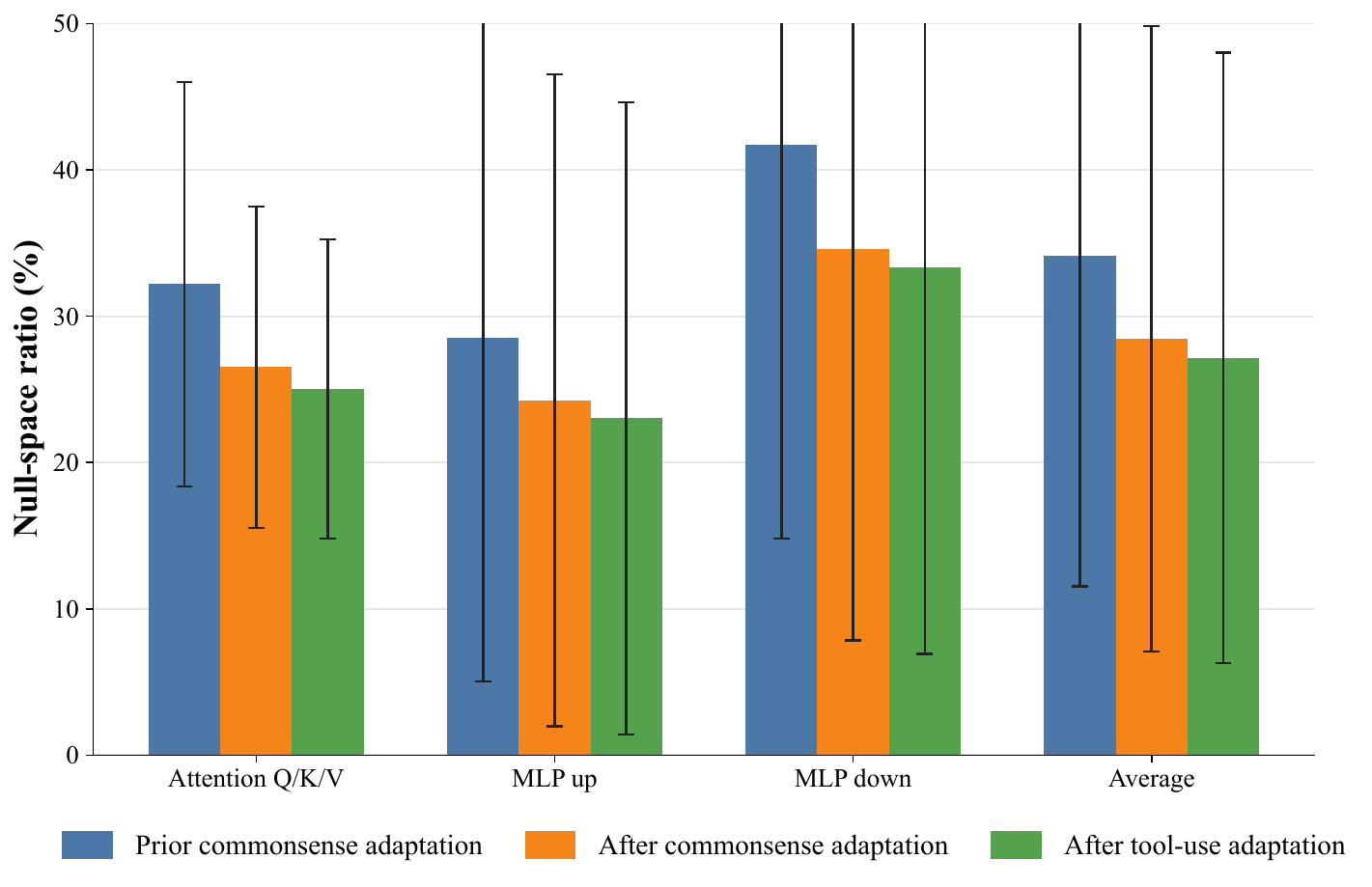}
    \caption{\textbf{Null-space capacity remains after sequential adaptation.} Mean null-space ratio across target module groups before commonsense fine-tuning, before tool-use adaptation, and after tool-use adaptation. The null-space ratio is defined as the retained basis dimension divided by the activation-space dimension. Error bars denote one standard deviation across layers. Although the null-space ratio decreases across the adaptation stages, the post-tool-use model retains substantial approximate null-space capacity across attention and MLP projections.}
    \label{fig:sequential_null_space_capacity}
\end{figure}

\section{Experiment Details}\label{app:appendix_experiment_details}

This appendix provides additional implementation details for reproducibility, including the training hyperparameters, datasets, prompt templates, and evaluation settings used in our experiments. All experiments were conducted on NVIDIA H100 GPUs with 80\,GB of memory.

\subsection{Hyperparameters Details}
The hyperparameters used for NB-LoRA are summarized in Table~\ref{tab:nb_lora_hyperparameters}.

\begin{table}[h]
\centering
\small
\caption{Training and null-space construction hyperparameters for NB-LoRA. \label{tab:nb_lora_hyperparameters}}
\begin{tabular}{lcc}
\toprule
Hyperparameter & Commonsense reasoning & Tool use \\
\midrule
Optimizer & \multicolumn{2}{c}{Adam} \\
Learning rate & $1 \times10^{-4}$ & $2\times10^{-4}$ \\
Epochs & \multicolumn{2}{c}{3} \\
Batch size & \multicolumn{2}{c}{16} \\
Maximum sequence length & 1024 & 4096 \\
LoRA rank $r$ & \multicolumn{2}{c}{16} \\
LoRA scaling $\alpha$ & \multicolumn{2}{c}{32} \\
LoRA dropout & \multicolumn{2}{c}{0.05} \\
Adapted modules & \multicolumn{2}{c}{Query, key, value, up-, and down-projections} \\
Anchor corpus & \multicolumn{2}{c}{MATH-lighteval (3B); GSM8K (1.5B)} \\
Number of anchor examples & \multicolumn{2}{c}{512} \\
Protected spectral energy & \multicolumn{2}{c}{99\%} \\
Reasoning tokens & \multicolumn{2}{c}{All tokens in the generated reasoning traces} \\
\bottomrule
\end{tabular}
\end{table}

\subsection{Datasets Details}
We provide additional details on the commonsense reasoning and tool-use datasets used for downstream adaptation, as well as MMLU-STEM, which is used only for held-out evaluation.
\ifdefined\preprintlayout
\begin{longtable}{@{}p{0.14\textwidth}p{\dimexpr0.86\textwidth-2\tabcolsep\relax}@{}}
\else
\begin{table}[H]
\fi
\caption{The prompt template of the commonsense reasoning datasets.}
\label{tab:input_template}

\ifdefined\preprintlayout
\\
\else
\begin{tabularx}{\textwidth}{lX}
\fi
\toprule
Dataset & Input Template \\
\midrule

ARC-C/E & 
\begin{minipage}{\linewidth}
\begin{tcolorbox}[colback=gray!10,colframe=black!75!black, left=1mm, right=1mm, top=1mm, bottom=1mm]
Please choose the correct answer to the question: [QUESTION] 

Answer1: [ANSWER\_1]

Answer2: [ANSWER\_2]

Answer3: [ANSWER\_3]

Answer4: [ANSWER\_4]

Answer format: answer1/answer2/answer3/answer4

the correct answer is [ANSWER]
\end{tcolorbox}
\end{minipage} \\
\midrule

BoolQ & 
\begin{minipage}{\linewidth}
\begin{tcolorbox}[colback=gray!10,colframe=black!75!black, left=1mm, right=1mm, top=1mm, bottom=1mm]
Please answer the following question with true or false, question: [QUESTION]

Answer format: true/false

the correct answer is [ANSWER]
\end{tcolorbox}
\end{minipage} \\
\midrule

HellaSwag & 
\begin{minipage}{\linewidth}
\begin{tcolorbox}[colback=gray!10,colframe=black!75!black, left=1mm, right=1mm, top=1mm, bottom=1mm]
Please choose the correct ending to complete the given sentence: [ACTIVITY\_LABEL]: [CONTEXT]

Ending1: [ENDING\_1] 

Ending2: [ENDING\_2] 

Ending3: [ENDING\_3] 

Ending4: [ENDING\_4] 

Answer format: ending1/ending2/ending3/ending4 

the correct answer is [ANSWER]
\end{tcolorbox}
\end{minipage} \\
\midrule

OpenBookQA & 
\begin{minipage}{\linewidth}
\begin{tcolorbox}[colback=gray!10,colframe=black!75!black, left=1mm, right=1mm, top=1mm, bottom=1mm]
Please choose the correct answer to the question: [QUESTION] 

Answer1: [ANSWER\_1]

Answer2: [ANSWER\_2]

Answer3: [ANSWER\_3]

Answer4: [ANSWER\_4]

Answer format: answer1/answer2/answer3/answer4

the correct answer is [ANSWER]
\end{tcolorbox}
\end{minipage} \\
\midrule

PIQA & 
\begin{minipage}{\linewidth}
\begin{tcolorbox}[colback=gray!10,colframe=black!75!black, left=1mm, right=1mm, top=1mm, bottom=1mm]
Please choose the correct solution to the question: [QUESTION]

Solution1: [SOLUTION\_1]

Solution2: [SOLUTION\_2]

Answer format: solution1/solution2 

the correct answer is [ANSWER]
\end{tcolorbox}
\end{minipage} \\
\midrule

SocialIQA & 
\begin{minipage}{\linewidth}
\begin{tcolorbox}[colback=gray!10,colframe=black!75!black, left=1mm, right=1mm, top=1mm, bottom=1mm]
Please choose the correct answer to the question: [QUESTION]

Answer1: [ANSWER\_1]

Answer2: [ANSWER\_2]

Answer3: [ANSWER\_3]

Answer format: answer1/answer2/answer3 

the correct answer is [ANSWER]
\end{tcolorbox}
\end{minipage} \\
\midrule

WinoGrande & 
\begin{minipage}{\linewidth}
\begin{tcolorbox}[colback=gray!10,colframe=black!75!black, left=1mm, right=1mm, top=1mm, bottom=1mm]
Please choose the correct answer to fill in the blank to complete the given sentence: [SENTENCE]

Option1: [OPTION\_1]

Option2: [OPTION\_2]

the correct answer is [ANSWER]
\end{tcolorbox}
\end{minipage} \\
\bottomrule

\ifdefined\preprintlayout
\end{longtable}
\else
\end{tabularx}
\end{table}
\fi

\subsubsection{Commonsense Reasoning Benchmarks}

The commonsense reasoning benchmarks contain approximately 170K training examples from eight tasks: ARC-C and ARC-E \citep{clark2018think}, BoolQ \citep{clark2019boolq}, HellaSwag \citep{zellers2019hellaswag}, OpenBookQA \citep{mihaylov2018can}, PIQA \citep{bisk2020piqa}, SocialIQA \citep{sap2019socialiqa}, and WinoGrande \citep{sakaguchi2021winogrande}.
We briefly describe these eight tasks below.
\begin{itemize}[leftmargin=*]
    \item \textbf{ARC-C/E} comprises the Challenge and Easy partitions of ARC, a collection of multiple-choice science questions drawn from grade-school examinations.
    \item \textbf{BoolQ} contains naturally occurring yes-or-no questions paired with passages that provide the information needed to answer them.
    \item \textbf{HellaSwag} asks the model to select the most plausible continuation of a given context from four candidate endings.
    \item \textbf{OpenBookQA} evaluates elementary science question answering that requires combining core scientific facts with broader commonsense knowledge.
    \item \textbf{PIQA} measures physical commonsense by asking the model to choose the more plausible of two solutions to an everyday problem.
    \item \textbf{SocialIQA} uses three-choice questions to assess reasoning about people's intentions, actions, reactions, and social consequences.
    \item \textbf{WinoGrande} evaluates commonsense and coreference reasoning through binary-choice fill-in-the-blank questions derived from ambiguous sentences.
\end{itemize}
The input template, i.e., prompt format for these datasets is detailed in Table \ref{tab:input_template}~\citep{hu2023llm,fang2025federated}.

\subsubsection{Tool-Use Tasks}

For the tool-use task, we use ToolAlpaca~\citep{tang2023toolalpaca}, in which the model receives a user request and natural-language documentation for a tool collection, and must select an API and produce its arguments as a JSON object. We retain only single-call examples because they provide clean, self-contained supervision, whereas later calls in multi-call trajectories may depend on unobserved intermediate tool outputs. The resulting subset contains 2,513 training examples and 62 evaluation examples. The generic prompt format is shown in Table~\ref{tab:tool_use_input_template}, and representative examples are shown in Table~\ref{tab:tool_use_examples}.

\begin{table}[H]
\caption{The prompt template of the tool-use adaptation dataset.}
\label{tab:tool_use_input_template}
\begin{tabularx}{\textwidth}{lX}
\toprule
Dataset & Input Template \\
\midrule
Tool use &
\begin{minipage}{\linewidth}
\begin{tcolorbox}[colback=gray!10,colframe=black!75!black, left=1mm, right=1mm, top=1mm, bottom=1mm]
\raggedright
Your task is to answer the user's question using available tools.

You have access to the following tools:

Name: [TOOL NAME]

Description: [TOOL DESCRIPTION]

Documentation:

[ACTION NAME]: [ACTION DESCRIPTION]

Parameters: [JSON PARAMETER SCHEMA]

Output: [OUTPUT DESCRIPTION]

Use the following format:

Thought: you should always think about what to do

Action: the action to take, which should be one of the available action names

Action Input: the input to the action in JSON format

Begin!

Question: [USER QUESTION]
\end{tcolorbox}
\end{minipage} \\
\bottomrule
\end{tabularx}
\end{table}

\begin{table}[H]
\caption{Representative tool-use examples from the Axolotl tool collection. Each gold response identifies the correct API and supplies its arguments as a JSON object.}
\label{tab:tool_use_examples}
\begin{tabularx}{\textwidth}{lX}
\toprule
Example & User Request and Gold Tool Call \\
\midrule
Random image &
\begin{minipage}{\linewidth}
\begin{tcolorbox}[colback=gray!10,colframe=black!75!black, left=1mm, right=1mm, top=1mm, bottom=1mm]
\raggedright
\textbf{Question:} Hey, can you show me a random picture of an axolotl?

\textbf{Action:} \texttt{getRandomAxolotlImage}

\textbf{Action Input:} \texttt{\{\}}
\end{tcolorbox}
\end{minipage} \\
\midrule
Filtered search &
\begin{minipage}{\linewidth}
\begin{tcolorbox}[colback=gray!10,colframe=black!75!black, left=1mm, right=1mm, top=1mm, bottom=1mm]
\raggedright
\textbf{Question:} I'm looking for an axolotl that is wild in color and medium in size. Can you help me find some pictures?

\textbf{Action:} \texttt{searchAxolotlImages}

\textbf{Action Input:} \texttt{\{"color": "wild", "gender": "", "size": "medium", "page": 1\}}
\end{tcolorbox}
\end{minipage} \\
\midrule
Fact lookup &
\begin{minipage}{\linewidth}
\begin{tcolorbox}[colback=gray!10,colframe=black!75!black, left=1mm, right=1mm, top=1mm, bottom=1mm]
\raggedright
\textbf{Question:} Can you tell me three interesting facts about axolotls' habitats?

\textbf{Action:} \texttt{getAxolotlFacts}

\textbf{Action Input:} \texttt{\{"category": "habitat", "limit": 3\}}
\end{tcolorbox}
\end{minipage} \\
\bottomrule
\end{tabularx}
\end{table}

\subsubsection{MMLU-STEM Benchmark}
MMLU~\citep{hendrycks2020measuring} is a multiple-choice benchmark spanning 57 academic and professional subjects. Its STEM subset contains 3,018 evaluation questions across 18 subjects in mathematics, physics, chemistry, biology, computer science, and engineering, ranging from elementary to college level. MMLU-STEM is used only for held-out evaluation; its examples are not used for downstream adaptation or null-space construction.

\section{Qualitative Case Study: Reasoning Retention and Downstream Adaptation}
\label{app:reasoning_examples}

We examine two complementary aspects of post-RL adaptation. First, using MATH-lighteval queries \citep{hendrycks2021measuring}, we compare the Qwen2.5-3B post-RL model before and after adaptation with Vanilla LoRA and NB-LoRA, thereby assessing retention of its original mathematical reasoning ability. Second, using HellaSwag queries \citep{zellers2019hellaswag}, we illustrate adaptation to a downstream commonsense task.
\subsection{MATH Reasoning Examples}


\paragraph{Query.}
How many vertical asymptotes does the graph of
$y=2/(x^2+x-6)$ have? Think step by step and place the final answer in
$\boxed{\text{answer}}$.

\paragraph{Reference answer.} $2$.

\begin{tcolorbox}[title={Post-RL reasoning model (correct)},
colback=gray!6,colframe=black!65,left=1mm,right=1mm,top=1mm,bottom=1mm]
To determine the number of vertical asymptotes, we need to find the values of
$x$ that make the denominator zero. First, factor the denominator:
\[
x^2+x-6=(x+3)(x-2).
\]
Setting it equal to zero gives $x=-3$ and $x=2$. These are potential vertical
asymptotes, and neither is canceled because the numerator is the nonzero
constant $2$. Therefore, the function has two vertical asymptotes, and the
final answer is $\boxed{2}$.
\end{tcolorbox}

\begin{tcolorbox}[title={After Vanilla LoRA (incorrect)},
colback=red!3,colframe=red!55!black,left=1mm,right=1mm,top=1mm,bottom=1mm]
\texttt{the correct answer is \textbackslash boxed\{6\}}

\end{tcolorbox}

\begin{tcolorbox}[title={After NB-LoRA (correct)},
colback=green!3,colframe=green!45!black,left=1mm,right=1mm,top=1mm,bottom=1mm]
The given function is a rational function. To find its vertical asymptotes, we
find the values of $x$ that make the denominator zero:
\[
x^2+x-6=(x+3)(x-2)=0.
\]
Thus, $x=-3$ and $x=2$, so the graph has two vertical asymptotes. The final
answer is $\boxed{2}$.
\end{tcolorbox}


\paragraph{Query.}
What is the smallest value of $x$ such that
$|5x-1|=|3x+2|$? Express the answer as a common fraction, think step by step,
and place the final answer in $\boxed{\text{answer}}$.

\paragraph{Reference answer.} $-1/8$.

\begin{tcolorbox}[title={Post-RL reasoning model (correct)},
colback=gray!6,colframe=black!65,left=1mm,right=1mm,top=1mm,bottom=1mm]
The expressions inside the absolute values change sign at $x=1/5$ and
$x=-2/3$, which divide the number line into three intervals.

For $x<-2/3$, both expressions are negative, so
\[
-(5x-1)=-(3x+2) \implies x=\frac{3}{2},
\]
which is outside this interval. For $-2/3<x<1/5$, the first expression is
negative and the second is positive, so
\[
-(5x-1)=3x+2 \implies x=-\frac{1}{8},
\]
which lies in this interval. For $x>1/5$, both expressions are positive, so
\[
5x-1=3x+2 \implies x=\frac{3}{2},
\]
which is valid in this interval. The smaller of the two valid solutions is
$-1/8$, so the final answer is $\boxed{-\frac{1}{8}}$.
\end{tcolorbox}

\begin{tcolorbox}[title={After Vanilla LoRA (incorrect)},
colback=red!3,colframe=red!55!black,left=1mm,right=1mm,top=1mm,bottom=1mm]
\texttt{the correct answer is \textbackslash boxed\{\textbackslash
frac\{3\}\{4\}\}}
\end{tcolorbox}

\begin{tcolorbox}[title={After NB-LoRA (correct)},
colback=green!3,colframe=green!45!black,left=1mm,right=1mm,top=1mm,bottom=1mm]
The critical points are $x=1/5$ and $x=-2/3$, yielding the intervals
$(-\infty,-2/3)$, $(-2/3,1/5)$, and $(1/5,\infty)$.

For $x<-2/3$, both expressions are negative:
\[
-(5x-1)=-(3x+2) \implies x=\frac{3}{2},
\]
which is not in the interval. For $-2/3<x<1/5$, the first expression is
negative and the second is positive:
\[
-(5x-1)=3x+2 \implies x=-\frac{1}{8},
\]
which is valid. For $x>1/5$, both expressions are positive:
\[
5x-1=3x+2 \implies x=\frac{3}{2},
\]
which is also valid. Hence, the smaller solution is $-1/8$, and the final
answer is $\boxed{-\frac{1}{8}}$.
\end{tcolorbox}
\subsection{HellaSwag Examples}


\paragraph{Query.}
A man is holding a pocket knife while sitting on some rocks in the
wilderness. Then he:
\begin{enumerate}[leftmargin=*,nosep]
    \item opens a can of oil, puts oil on the knife, presses it through the
          can, then cuts pieces from sandwiches;
    \item takes a small stone from the flowing river and smashes it on another
          stone;
    \item uses the knife to shave his leg; or
    \item sands the rocks and tops them using strong pressure.
\end{enumerate}

\paragraph{Reference answer.} \texttt{ending2}.

\begin{tcolorbox}[title={Post-RL reasoning model (incorrect)},
colback=gray!6,colframe=black!65,left=1mm,right=1mm,top=1mm,bottom=1mm]
\texttt{The correct ending to complete the given sentence is: Ending3: uses
the knife to shave his leg.}
\end{tcolorbox}

\begin{tcolorbox}[title={After Vanilla LoRA (correct)},
colback=red!3,colframe=red!55!black,left=1mm,right=1mm,top=1mm,bottom=1mm]
\texttt{The correct answer is ending2.}
\end{tcolorbox}

\begin{tcolorbox}[title={After NB-LoRA (correct)},
colback=green!3,colframe=green!45!black,left=1mm,right=1mm,top=1mm,bottom=1mm]
\texttt{The correct answer is ending2, because it logically follows.}
\end{tcolorbox}


\paragraph{Query.}
A boy is running down a track. The boy:
\begin{enumerate}[leftmargin=*,nosep]
    \item runs into a car;
    \item gets in a mat;
    \item lifts his body above the height of a pole; or
    \item stands on his hands and springs.
\end{enumerate}

\paragraph{Reference answer.} \texttt{ending3}.

\begin{tcolorbox}[title={Post-RL reasoning model (incorrect label)},
colback=gray!6,colframe=black!65,left=1mm,right=1mm,top=1mm,bottom=1mm]
\texttt{The correct ending to complete the given sentence is: Ending1: lifts
his body above the height of a pole.}
\end{tcolorbox}

\begin{tcolorbox}[title={After Vanilla LoRA (correct)},
colback=red!3,colframe=red!55!black,left=1mm,right=1mm,top=1mm,bottom=1mm]
\texttt{The correct answer is ending3.}
\end{tcolorbox}

\begin{tcolorbox}[title={After NB-LoRA (correct)},
colback=green!3,colframe=green!45!black,left=1mm,right=1mm,top=1mm,bottom=1mm]
\texttt{The correct answer is ending3, because it logically follows.}
\end{tcolorbox}

\subsection{Summary}

Across the MATH examples, NB-LoRA retains the post-RL model's multi-step reasoning ability and required answer format after downstream adaptation, whereas Vanilla LoRA produces incorrect answers without an explicit reasoning process. The HellaSwag examples complement this analysis by showing that both adaptation methods acquire the downstream commonsense task and select the correct continuation in both cases. Taken together, these examples illustrate that NB-LoRA achieves the intended balance between preserving the original reasoning ability and acquiring a new downstream capability.
\end{document}